\documentclass[11pt]{article}

\RequirePackage{epsfig}
\RequirePackage{amssymb}
\RequirePackage{natbib}
\RequirePackage{graphicx}

\if@abbrvbib
\else
\fi
\usepackage[colorlinks = true,linkcolor = blue, citecolor = blue, urlcolor = blue]{hyperref}

 \setcitestyle{round}
\usepackage[ruled,vlined]{algorithm2e}
\usepackage{algpseudocode}
\usepackage{dsfont}
\usepackage{mathrsfs}
\usepackage{amsmath}   
\usepackage{mathtools}
\usepackage{float}
\usepackage{color}
\usepackage{xspace}
\usepackage{dsfont}
\usepackage{bm}
\usepackage{bbm}
\usepackage{enumitem}

\usepackage[capitalise]{cleveref}

\usepackage{enumitem}
\setlist[enumerate,1]{leftmargin=*,wide=0em, itemsep=0pt,topsep=3pt, label = {\bfseries \arabic*.}}
\setlist[itemize,1]{leftmargin=*,wide=0em, itemsep=0pt,topsep=3pt}

\crefname{equation}{}{}
\crefname{figure}{Figure}{Figure}
\creflabelformat{equation}{\textup{(#2#1#3)}}

\newcommand{\dd}{\mathrm{d}}
\renewcommand{\v}[1]{\boldsymbol{#1}}
\newcommand{\m}[1]{\mathbf{#1}}
\newcommand{\bb}[1]{\mathbb{#1}}
\newcommand{\gvn}{\, \vert \,}

\newcommand{\MMTV}{\mathrm{MMTV}}
\newcommand{\MMD}{\mathrm{MMD}}
\newcommand{\TV}{d_{\mathrm{TV}}}
\newcommand{\dotprod}[1]{\langle #1 \rangle}

\DeclareMathOperator*{\argmin}{arg\,min}
\newcommand{\reals}{\mathbb{R}}

\newcommand{\df}{\mathrm{d}}
\algrenewcommand\algorithmicrequire{\textbf{Input:}}
\algrenewcommand\algorithmicensure{\textbf{Output:}}
\renewcommand{\triangleq}{\coloneqq}
\newcommand{\prox}{\boldsymbol{{\rm prox}}}

\newcommand{\BlackBox}{\rule{1.5ex}{1.5ex}}  
\ifdefined\proof
    \renewenvironment{proof}{\par\noindent{\bf Proof\ }}{\hfill\BlackBox\\[2mm]}
\else
    \newenvironment{proof}{\par\noindent{\bf Proof\ }}{\hfill\BlackBox\\[2mm]}
\fi
 
\newtheorem{theorem}{Theorem}
\newtheorem{lemma}[theorem]{Lemma} 
\newtheorem{proposition}[theorem]{Proposition} 
\newtheorem{remark}[theorem]{Remark}

\newtheorem{definition}[theorem]{Definition}

\newtheorem{assumption}{Assumption}
\usepackage{fullpage}

\usepackage{lastpage}

\begin{document}

\title{Implicit Langevin Algorithms \\ for Sampling From Log-concave Densities}

\author{Liam Hodgkinson \\     Department of Statistics, UC Berkeley, Berkeley, CA, 94720, USA \\
       International Computer Science Institute, Berkeley, CA, 94704, USA  \\ 
       \texttt{liam.hodgkinson@berkeley.edu} \and
       Robert Salomone \\
       Centre for Data Science, Queensland University of Technology\\
        Brisbane, QLD, 4001, Australia \\ 
       \texttt{robert.salomone@qut.edu.au} \and
       Fred Roosta \\
       School of Mathematics and Physics, University of Queensland\\ St Lucia QLD 4067, Australia \\
       International Computer Science Institute, Berkeley, CA 94704, US \\ 
       \texttt{fred.roosta@uq.edu.au}
       }

\date{}       
\maketitle

\begin{abstract}%
For sampling from a log-concave density, we study implicit integrators resulting from $\theta$-method discretization of the overdamped Langevin diffusion stochastic differential equation. Theoretical and algorithmic properties of the resulting sampling methods for $ \theta \in [0,1] $ and a range of step sizes are established. Our results generalize and extend prior works in several directions. In particular, for $\theta\ge1/2$, we prove geometric ergodicity and stability of the resulting methods for all step sizes. We show that obtaining subsequent samples amounts to solving a strongly-convex optimization problem, which is readily achievable using one of numerous existing methods. Numerical examples supporting our theoretical analysis are also presented.  \\ 
\end{abstract}

\section{Introduction}
\label{sec:Intro}
Effectively sampling from arbitrary unnormalized probability distributions is a fundamental aspect of the Monte Carlo method, and is central in Bayesian inference. The most common cases involve probability densities $\pi$ with support on all of $\mathbb{R}^d$, which can be written in the unnormalized form as
\begin{align*}
\pi(\v x) \propto \exp(-f(\v x)),\qquad \v x \in \mathbb{R}^d.
\end{align*}
The sampling problem concerns the construction of a set of points $\{\v X_k\}$ whose empirical distribution approaches $\pi$ in some appropriate sense. 
A standard approach is \emph{Markov Chain Monte Carlo} (MCMC), in which approximate sampling from $\pi$ is accomplished by simulating a $\pi$-ergodic Markov chain. By the ergodic theorem, this provides consistent Monte Carlo estimators for expectations involving the density $\pi$. The most popular approach to generate such a set of points is the \emph{Metropolis-Hastings} algorithm \citep{Hastings1970}, which constructs a $\pi$-ergodic Markov chain by generating a proposal from a given transition kernel and implements an acceptance criterion for these proposals; see \citet{robert1999monte} for an overview of such methods. While geometric rates of convergence (geometric ergodicity) can be guaranteed in a wide variety of settings, performance is highly susceptible to the underlying proposal. 
However, the effectiveness of Metropolis-Hastings methods diminish in higher dimensions, as step sizes must be scaled inversely with dimension, making rapid exploration of the space unlikely; see for example, \citet{roberts2001optimal}.

Many of these issues lie with the steadfast requirement of consistency: that 
the sample empirical distribution should \emph{asymptotically} be the same as $\pi$.  Ensuring this requirement in turn can result in incurring serious penalty to the mixing rate of the chain. However, when seeking a fixed (finite) number of samples, which is almost always the case in practice, consistency is not necessarily a decisive property. Therefore, it has recently become popular to consider rapidly converging Markov chains whose stationary distributions are only \emph{approximations} to $\pi$ with a bias of adjustable size \citep{dalalyan2017further,dalalyan2017theoretical,wibisono2018sampling,cheng2018sharp,cheng2018convergence}. While the resulting Monte Carlo estimator is no longer consistent, it will often have dramatically smaller variance. This is an example of a \emph{bias-variance tradeoff}, where a biased method can require significantly less computational effort to reach the same mean-squared error as an asymptotically unbiased Metropolis-Hastings chain \citep{korattikara2014austerity}. 

The most studied of these methods is the \emph{unadjusted Langevin algorithm} (ULA), seen in \citet{roberts1996}, which is constructed by considering the overdamped Langevin diffusion equation, given by the stochastic differential equation (SDE)
\begin{align} 
\label{eq:overdamped_Langevin}
\v L_0 \sim \pi_0,  \qquad \dd \v L_{t} = -\frac{1}{2}\nabla f(\v L_{t}) \dd t+ \dd \v W_{t}, 
\end{align}
and employing the forward Euler integrator, also known as Euler--Maruyama approximation \citep{kloeden2013numerical}, to obtain iterates of the form
\begin{align}
\label{eq:overdamped_Langevin_FE}
\v X_{k+1} = \v X_k - \frac{h}{2} \nabla f(\v X_k) + \sqrt{h} \v Z_k.
\end{align}
Here, $\v W_{t}$ denotes $d$-dimensional standard Brownian motion, $\pi_0$ is some arbitrary (possibly deterministic) initial distribution, each $\v Z_k$ is an independent standard $d$-dimensional normal random vector, and $h$ is the \emph{step size} parameter representing the temporal mesh size of the Euler method. Since \cref{eq:overdamped_Langevin_FE} is explicitly defined, it is often referred to as \emph{explicit} Euler scheme.
This main interest in \cref{eq:overdamped_Langevin} lies in the well known fact that, under certain mild conditions and regardless of $\pi_0$, for any $t > 0$ the distribution of $\v L_t$ is absolutely continuous (so we may consider its density $\pi_t$ on $\mathbb{R}^d$), and $\v L_t$ is an ergodic Markov process with limiting distribution $\pi$, that is, $\pi_t(\v x) \to \pi(\v x)$ as $t \to \infty$ for all $\v x \in \mathbb{R}^d$ \citep{kolmogorov1937}.  

However, unlike the Langevin SDE, the distribution of samples obtained from ULA \cref{eq:overdamped_Langevin_FE} will, generally speaking, not converge to $\pi$ as $t\to\infty$. More precisely, ULA is an asymptotically \emph{biased} sampling algorithm, with corresponding bias proportional to step size (temporal mesh size). Despite this, in situations where MCMC fails to perform well, for example, high-dimensional problems, ULA can provide approximate samples from the target density with acceptable accuracy \citep{Durmus2016}.

The theoretical properties of ULA, including geometric ergodicity \citep{hansen2003geometric,roberts1996}, and performance in high dimensions \citep{Durmus2016} are well understood. Of particular relevance to us is the recent work of \citet{dalalyan2017further}, \citet{dalalyan2017theoretical}, and \citet{durmus2017nonasymptotic}, 
concerning the stability of ULA.
Although it does not possess a single technical definition, stability of stochastic
processes is often well-understood conceptually---some common characterizations include
non-evanescence and Harris recurrence \citep[p.\ 15]{meyn2012markov}. To establish stability, the aforementioned works develop theoretical guarantees in the form of error bounds on the 2-Wasserstein metric between iterates of ULA and the target distribution. Doing so gives conditions under which ULA is bounded in probability, which in turn implies non-evanescence \citep[Proposition 12.1.1]{meyn2012markov}, and Harris recurrence, of the corresponding Markov chain \citep[Theorem 9.2.2]{meyn2012markov}. 
The inexact case, where $\nabla f$ is approximated to within an absolute tolerance, is also considered \citep{dalalyan2017user}. Some alternative unadjusted explicit methods have also been considered; these are usually derived using other diffusions whose stationary distributions can also be prescribed \citep{cheng2017underdamped}.

As a direct result of the explicit nature of the underlying discretization scheme, the main issue with ULA-type algorithms is that they are stable only up to a fixed step size, beyond which the chain is no longer ergodic. 
In fact, \citet{roberts1996} actively discourage the use of ULA for this reason, and show that ULA may be transient for large step sizes.  Stability is an essential concept when designing and analyzing methods for the numerical integration of continuous-time differential equations \citep{ascher2008numerical}. In some cases, this step size must be taken extremely small to remain stable. This becomes a major hindrance to the performance of the method in practice. Drawing comparisons to the theory of ordinary differential equations (ODEs) by dropping the stochastic term, in these cases, the Langevin diffusion is said to be \emph{stiff} \citep{ascher1998computer}. Ill-conditioned problems, such as sampling from any multivariate normal distribution with a covariance matrix possessing a large condition number \citep[\S2.6.2]{golub2012matrix},  are likely to induce a stiff Langevin diffusion %
\citep[\S6.2]{lambert1991numerical}. %
The negative side-effects associated with ill-conditioning as well as the restrictions on step size are often only exacerbated in high-dimensional problems. 

In this light, a natural alternative to using explicit schemes with careful choice of step size is to consider \emph{implicit} variants. From the established theory of numerical solutions of ODEs, it is well-known that implicit integrators have larger regions of stability than explicit alternatives, that is, one can take larger steps without unboundedly magnifying the underlying discretization errors \citep{ascher2008numerical}. 
Motivated by this, we can instead consider the $\theta$-method scheme~\citep[p.\ 84]{ascher2008numerical}, which when applied to Langevin dynamics \cref{eq:overdamped_Langevin}, yields general iterations of form 
\begin{align}
\label{eq:overdamped_Langevin_theta_method_0}
\v X_{k+1} = \v X_k - \frac{h}{2}\big[\theta \nabla f(\v X_{k+1}) + (1-\theta) \nabla f(\v X_k) \big] + \sqrt{h} \v Z_k,
\end{align}
for some $ \theta = [0,1] $. 
The special cases of $\theta = 0, 1$ and $1/2$ correspond to forward, backward, and trapezoidal integrators, respectively. Of course, for $\theta = 0$, \cref{eq:overdamped_Langevin_theta_method_0} reduces to the explicit Euler scheme \cref{eq:overdamped_Langevin_FE}. As the choice of $\theta \in (0,1]$ define the endpoint in an implicit way, such integrators are often referred to as {\em implicit}. 

To our knowledge, there have only been a handful of efforts to study the properties of sampling algorithms obtained from such implicit schemes. A universal analysis of sampling schemes based on general Langevin diffusion was conducted in \citet{mattingly2002ergodicity}. There, it was shown that the implicit Euler scheme, and other numerical methods satisfying a certain minorization condition are geometrically ergodic for sufficiently small step sizes under the assumption that $f$ is `essentially quadratic'. In a more focused analysis, \citet{Casella2011} investigated the ergodic properties of a few implicit schemes (including the $\theta$-method scheme) 
for a restricted family of one-dimensional super-Gaussian target densities.
They found that, in this setting, the $\theta$-method results in a geometrically ergodic chain for \emph{any} step size $h > 0$, provided that $\theta \geq 1/2$, and suggested the same might be true in higher dimensions. Under slightly weaker assumptions than strong convexity, \citet{kopec2014weak} conducted a weak backward error analysis providing error bounds on the expectation of the fully implicit Euler scheme ($\theta = 1$) with respect to suitable test functions. More recently, \citet{wibisono2018sampling} considered the $\theta=1/2$ case and provided a rate of convergence of the scheme towards its biased stationary distribution under the 2-Wasserstein metric, assuming strong convexity and small step sizes. Despite these efforts, it is still unclear how implicit schemes compare with explicit schemes more generally for large step sizes, and what the effect of $\theta$ is on the bias of the method.

The aim of this work is to study the $\theta$-method sampling scheme \cref{eq:overdamped_Langevin_theta_method_0} for all $\theta \in (0,1]$, as it applies to the relevant case of strongly log-concave distributions{{, particularly when $f$ is a strongly convex function and $\nabla f$ is Lipschitz continuous.}}
Such distributions arise frequently in Bayesian regression problems \citep{bishop2003bayesian}, for example  generalized linear models (GLMs) with a Gaussian prior \citep{chatfield2010introduction}. 

\subsection{Contributions} To those ends, the contributions of this work are as follows:
\begin{enumerate}

\item We show that the transition density associated with \cref{eq:overdamped_Langevin_theta_method_0} has a closed form solution. Then, using this, we establish conditions for geometric ergodicity, in terms of $\theta$, the step size $h$, Lipschitz continuity, tail behaviour, and semi-convexity of $f$ (Theorem \ref{thm:GeoErgodic}). By doing so, we show stability of the $\theta$-method scheme in multivariate settings for \emph{any} step size when $\theta \geq 1/2$ and $f$ is strongly convex, proving the conjecture by \citet{Casella2011}.

\item We provide non-asymptotic theoretical guarantees for long-time behaviour of \cref{eq:overdamped_Langevin_theta_method_0}, which extend those of \citet[Theorem 1]{dalalyan2017user} to the general implicit case (Theorem \ref{thm:MidGuarantee}). 

\item As for $ \theta > 0 $, iterations of \cref{eq:overdamped_Langevin_theta_method_0} involve solving a non-linear equation, we study the effect of inexact solutions of the underlying sub-problems. We propose practically computable termination criteria and quantify the effect of approximating each iterate on the convergence rate and long-term bias of the chain with this {{criterion}}. 

\item We establish large step size asymptotics for $\theta > 0$ via a central limit theorem as $h \to \infty$ (Theorem \ref{thm:GaussApprox}). As a consequence,
we develop an effective default heuristic choice of step size.

\item Finally, we demonstrate the empirical performance of the implicit $\theta$-method scheme in a series of numerical experiments; namely sampling from high-dimensional Gaussian distributions, and the posterior density of a Bayesian logistic regression problem involving a real data set.
\end{enumerate}
Proofs of all results can be found in \cref{sec:proof}.

\paragraph{Notation.}
In the sequel, vectors and matrices are denoted by bold lowercase and Romanized bold uppercase letters,  
for example, $\v v$ and $\m V$, respectively. We denote the identity matrix by $\m I$. Regular lower-case and upper-case letters, such as s $ m $  and $ M $, are used to denote scalar constants. Random vectors are denoted by italicized bold uppercase letters, such as $\v X$.
For two symmetric matrices $\m A$ and $\m B$, $\m A \succeq \m B$ indicates that $\m A-\m B$ is symmetric positive semi-definite. For vectors, we let $\|\cdot\|$ denote the Euclidean norm of appropriate
dimension, and $\|\cdot\|_{L^2}$ denote the $L^2$ norm acting on random vectors, that is, $\|\v X\|_{L^2}^2 \triangleq \mathbb{E}\|\v X\|^2$. For matrices, $ \|\,\cdot\,\|_{2} $ denotes the spectral norm.

\section{Implicit Langevin Algorithm (ILA)}
\label{sec:SILA}
In this section,  we establish conditions under which the sequence of $\theta$-method iterates \cref{eq:overdamped_Langevin_theta_method_0}  form a Markov chain that is geometrically ergodic. For this, we impose the following assumption on the smoothness of $ f $, which ensures the existence of a unique solution to \cref{eq:overdamped_Langevin}; see \citet[Theorem 2.4--3.1]{ikeda2014stochastic}. 

\begin{assumption}
\label{ass:Lipschitz}
The function $f \in \mathcal{C}^2$ (it is twice continuously differentiable), and $\nabla f$ is $M$-Lipschitz, that is,
\begin{align*}
\|\nabla f(\v x) - \nabla f(\v y)\|\leq M\|\v x - \v y\|,\qquad \mbox{for any }\v x, \v y \in \mathbb{R}^d.
\end{align*}
\end{assumption}

Under \cref{ass:Lipschitz}, \citet{dalalyan2017theoretical} shows that if $h < 4/M$, iterations of ULA \cref{eq:overdamped_Langevin_FE} are {{bounded in probability (in other words, the algorithm is \emph{stable})}}. 
However, this restriction is a fundamental disadvantage of ULA. If $M$ is particularly large, as might be the case in ill-conditioned problems and in high dimensions where the ULA is commonly applied, then the step size must be taken very small, which results in slow mixing time of the chain and high autocorrelation of the samples. 
In sharp contrast, we now show that for appropriate choice of $\theta$, \cref{eq:overdamped_Langevin_theta_method_0} does not suffer from this restriction.

The process of obtaining samples by iterating \cref{eq:overdamped_Langevin_theta_method_0} is outlined in \cref{alg:sila_exact}. For brevity, we henceforth refer to this procedure as the implicit Langevin algorithm (ILA). Note that (\ref{eq:overdamped_Langevin_theta_method_0}) can be rewritten as
\[({\cal I} + \tfrac{1}{2}h\theta  \nabla f) (\v X_{k+1}) = \v X_k - \tfrac{1}{2}h (1-\theta)\nabla f(\v X_k) + \sqrt{h}\v Z_k,\]
where ${\cal I}$ denotes the identity mapping. 
Assuming that {{${\cal I} + \frac12 h \theta \nabla f$ is a strictly monotone operator, that is ${\cal I} + \frac12 h \theta \nabla f$ is globally invertible, \cref{eq:overdamped_Langevin_theta_method_0} admits a unique solution as }}
\begin{equation}
\label{eq:overdamped_Langevin_theta_inter}
\v X_{k+1} =  ({\cal I} + \tfrac{1}{2}h\theta  \nabla f )^{-1} \left[\v X_k - \tfrac{1}{2}h(1-\theta)\nabla f(\v X_k) + \sqrt{h}\v Z_k\right].
\end{equation} 
Conditions under which the procedure \cref{eq:overdamped_Langevin_theta_inter} is guaranteed to be well-defined are discussed in \S\ref{sec:overdamped_Langevin_theta_method_metropolis}{{; see \citet{rockafellar1976monotone} for a thorough treatment of monotone operators and their application in proximal point algorithms.}}
For the time being, it is also assumed in Algorithm \ref{alg:sila_exact} that (\ref{eq:overdamped_Langevin_theta_method_0}) can be solved \emph{exactly}. The discussion of inexact solutions is relegated to \S\ref{sec:Proximal}. 

\begin{algorithm}[htbp]
	\SetAlgoLined
	\SetKwInOut{Input}{Input}\SetKwInOut{Output}{output}
	\Input{- Initial value $\v X_{0} = \v x_0 \in \mathbb{R}^d$  \\
		- Number of samples $n$ \\
		- Step size $h > 0$ \\
		- $\theta$-method parameter $\theta \in (0,1]$}
	\For{$k=0,1,\ldots,n$}{
		Draw $\v Z_k \sim \mathcal{N}(0,\m I)$ \\
		Solve \cref{eq:overdamped_Langevin_theta_method_0} to obtain $\v X_{k+1}$
	}
	\caption{Implicit Langevin Algorithm (ILA)}
	\label{alg:sila_exact}
\end{algorithm}

\subsection{Theoretical analysis of Algorithm \ref{alg:sila_exact}}
\label{sec:overdamped_Langevin_theta_method_metropolis}
In this section, we establish sufficient conditions for the geometric ergodicity of the sequence of iterates generated from \cref{alg:sila_exact}.
To conduct such an analysis, we require the transition kernel density $p(\v y \gvn \v x) $ induced from~\eqref{eq:overdamped_Langevin_theta_method_0}. In general, this is only implicitly defined, however, assuming ${\cal I}+\frac12 h \theta \nabla f$ is globally invertible,
$p(\v y \gvn \v x) $ is nevertheless available in closed form. 
Assuming that $f \in \mathcal{C}^2$, 
this is true whenever $h$ is chosen so that
\begin{equation}
\label{eq:overdamped_Langevin_cond_h}
\m I + \frac{h\theta}{2}\nabla^{2} f(\v x) \succ 0, \qquad \mbox{for all } \v x \in \mathbb{R}^d.
\end{equation}
Therefore, at the very least, for (\ref{eq:overdamped_Langevin_theta_method_0}) to be well-defined as a sampling method, we require $f$ to be {\em semi-convex}, that is, there exists some $\gamma > 0$ such that $\nabla^2 f(\v x) + \gamma \m I$ is positive-semidefinite for all $\v x \in \mathbb{R}^d$. For example, under Assumption \ref{ass:Lipschitz}, $f$ is $M$-semi-convex and \eqref{eq:overdamped_Langevin_cond_h} holds if $h < \frac{2}{\theta M}$. This restriction on step size can be removed entirely if $f$ is assumed to be convex.

From (\ref{eq:overdamped_Langevin_theta_inter}), for fixed $\v x \in \mathbb{R}^d$, note that $p(\, \cdot \, \vert \, \v x)$ is the probability density function of the random variable
\[
\v Y = (\mathcal{I} + \tfrac12 h \theta \nabla f)^{-1} (\v x - \tfrac12 h(1-\theta) \nabla f(\v x) + \sqrt{h} \v Z),
\]
where $\v Z \sim \mathcal{N}(\v 0, \m I)$. As this is  an invertible transformation of a standard Gaussian random vector, by the change of variables theorem (see for example,~\citet[Proposition 1.8]{shao2008mathematical}), we have 
\begin{equation}
\label{eq:overdamped_kernel}
	p(\v y \mid \v x)  = \left|\det\left(\m I + \frac{h \theta}{2} \nabla^{2} f(\v y)\right)\right| \phi\left(\v y + \frac{h \theta}{2} \nabla f(\v y) \, ; \, \v x - \frac{h(1-\theta)}{2} \nabla f(\v x), \;  h \m I \right),
\end{equation}
where $\phi(\, \cdot \,; \v \mu,  \v\Sigma )$ is the density of a multivariate Gaussian distribution with mean $ \v \mu  $ and covariance matrix $ \v\Sigma $,  and `$ \det $' denotes the determinant of a matrix. 

It can be seen from (\ref{eq:overdamped_kernel}) that increasing $\theta$ (or $h$ when $\theta > 0$) alters the landscape of the transition density, and the shape of its level-sets. To illustrate this, \cref{fig:transition_kernel} depicts the contour plots of the transition kernel \cref{eq:overdamped_kernel} for an anisotropic example problem with differing $\theta$ and initial state for the same step size. It can be seen that the case with $\theta = 0$ (ULA) results in an isotropic proposal in all situations, whereas other choices of $\theta$ (implicit methods) yield proposal densities that can better adapt to the anisotropic target density. 

\begin{figure}[t]
\includegraphics[width=\textwidth]{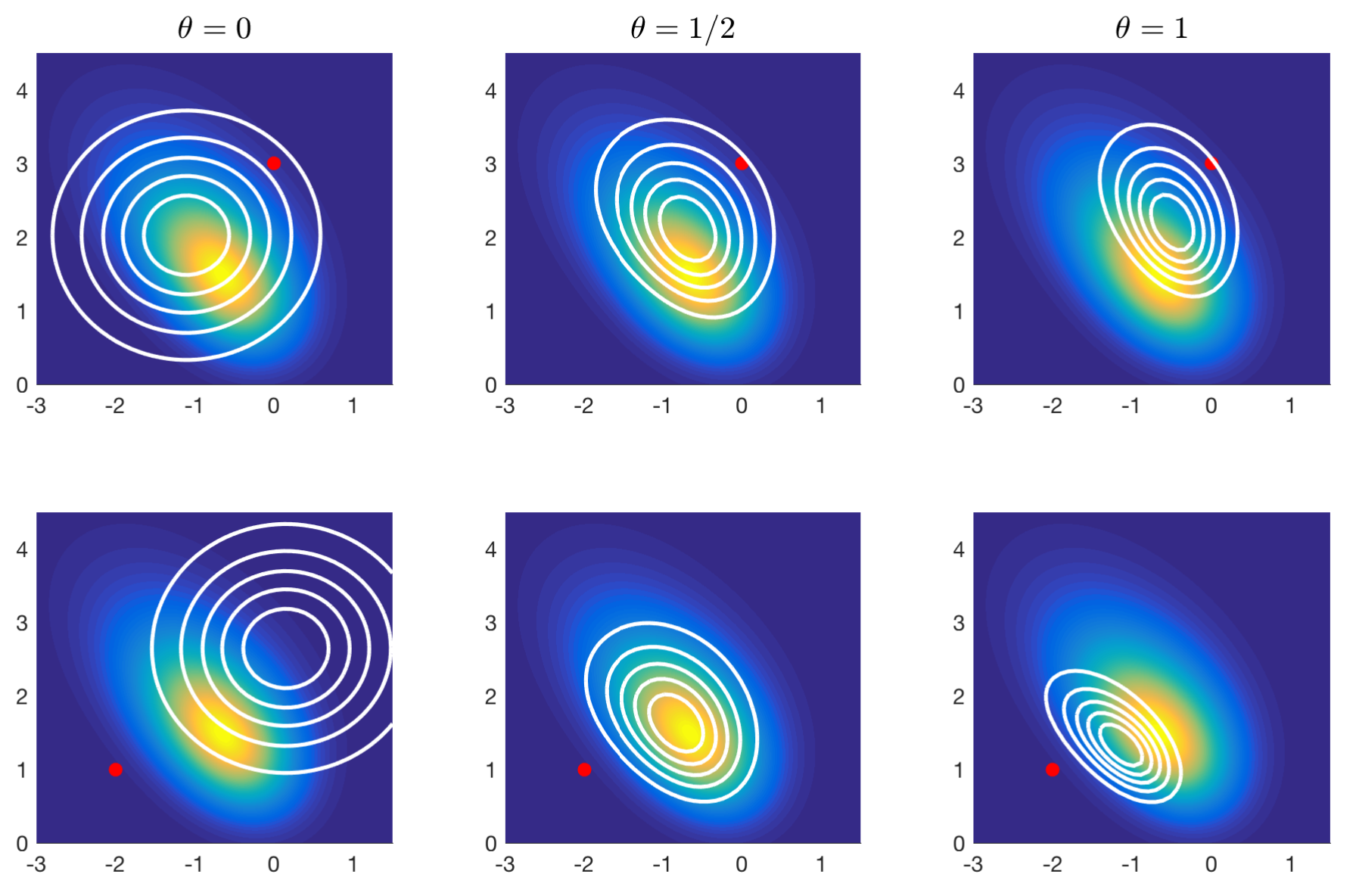}
\caption{Transition kernel for an example density and large step size on a two--dimensional Bayesian Logistic Regression example for large $h$. While the traditional $\theta = 0$ case (ULA) imposes isotropic proposals, for other choices of $\theta$, the proposal density adapts to the anisotropic target density.
\label{fig:transition_kernel}}
\end{figure}

With an explicit expression for the transition density, we can investigate the stability
of the iterates given by \cref{alg:sila_exact}. The most
convenient way of doing this is by demonstrating geometric ergodicity of the chain induced by the transition kernel (\ref{eq:overdamped_kernel}). Recall that a Markov chain with $n$-fold transition kernel $p_n(\v y \vert \v x)$ is said to be \emph{geometrically ergodic} toward an invariant density $\nu(\cdot)$ if there exist constants $C > 0$ and $0 < \rho < 1$ such that
\begin{align*}
\sup_{\v x \in \mathbb{R}^d} \int_{\mathbb{R}^d} |p_n(\v y \gvn \v x) - \nu(\v y)|\dd \v y \leq C \rho^n,\qquad \mbox{for all }n=1,2,\dots.
\end{align*}
Similarly, a diffusion process $\v X_t$ with transition kernel density $p_t(\cdot \, \vert \,  \cdot)$ is said to be \emph{exponentially ergodic} towards an invariant density $\nu(\v x)$ if there exist constants $C, \lambda > 0$ such that
\begin{align*}
\sup_{\v x \in \mathbb{R}^d} \int_{\mathbb{R}^d} |p_t(\v y \gvn \v x) - \nu(\v y)|\dd \v y \leq C e^{-\lambda t},\qquad \mbox{for all }t > 0.
\end{align*}
It was shown in~\citet[Eqn.\ (12)]{hansen2003geometric} that the overdamped Langevin diffusion (\ref{eq:overdamped_Langevin}) is exponentially ergodic provided the following assumption on $f$ holds:
\begin{assumption}
\label{ass:SecantCond}
$m \coloneqq \liminf_{\|\v x\|\to\infty} \displaystyle\frac{\langle\nabla f(\v x), \v x\rangle}{\|\v x\|^2} > 0.$
\end{assumption}
Intuitively, \cref{ass:SecantCond} imposes super-Gaussian tails of the target distribution. 
Under Assumptions \ref{ass:Lipschitz} and \ref{ass:SecantCond}, for any $ \v x $ and $ \v y $, there exists a constant $c(\v y) \geq 0$ depending on $ \v y $, such that (Lemma \ref{lem:AssumpConnection})
\begin{equation}
\label{eq:SecantCondConseq}
\langle \nabla f(\v x) - \nabla f(\v y), \, \v x - \v y \rangle \geq m \|\v x - \v y\|^2 - c(\v y), \qquad \mbox{for every } \v x \in \mathbb{R}^d.
\end{equation}
\cref{ass:SecantCond} is not new, having also appeared in \citet{kopec2014weak}, and appears
to be among the weakest assumptions one can make to effectively study these implicit schemes. Clearly, \cref{eq:SecantCondConseq} is a significantly weaker condition than strong convexity:
\begin{assumption}
\label{ass:Convexity}
The function $f \in \mathcal{C}^2(\mathbb{R}^d)$ is $m$-strongly convex, that is, there exists $ 0< m < \infty $ such that 
\[
\langle \nabla f(\v x) - \nabla f(\v y), \v x - \v y\rangle \geq m \|\v x - \v y\|^2,\qquad \mbox{for any }\v x,\v y\in\mathbb{R}^d.
\]
\end{assumption}

{{In fact, it is straightforward to show that if Assumption \ref{ass:SecantCond} holds, then $m$ is necessarily common in the two inequalities.}}
Furthermore, we remark that under Assumptions \ref{ass:Lipschitz} and \ref{ass:Convexity}, the spectrum of every Hessian matrix $\nabla^2 f(\v x)$ is controlled to be within $[m,M] \subset (0,\infty)$. 
Strong convexity is quite a natural assumption in Bayesian regression problems, as it can be guaranteed for the class of GLMs with Gaussian priors \citep{dasgupta2011probability}. 

Under Assumptions \ref{ass:Lipschitz} and \ref{ass:SecantCond}, we can now establish geometric ergodicity of the $\theta$-method scheme under certain conditions on $\theta$ and $h$ (Theorem \ref{thm:GeoErgodic}). 
\begin{theorem}
\label{thm:GeoErgodic}
For $f$ satisfying Assumptions \ref{ass:Lipschitz} and \ref{ass:SecantCond} that is $\gamma$-semi-convex, the iterates of the $\theta$-method scheme with associated transition kernel (\ref{eq:overdamped_kernel}) form a geometrically ergodic chain when $h < \frac{2}{\theta\gamma}$ provided we also have either
$\theta \geq 1/2$, or both $\theta < 1/2$ and $h < 4 m / [M^2 (1 - 2\theta)]$.
\end{theorem}
While \cref{thm:GeoErgodic} establishes the geometric ergodicity of the chain towards \emph{some}
stationary distribution, in general, that distribution need not necessarily be $\pi$. Nevertheless, under Assumptions \ref{ass:Lipschitz} and \ref{ass:Convexity}, we have established that the $\theta$-method discretization of the overdamped Langevin equation is stable for \emph{any} step size, provided $\theta \geq 1/2$. For these choices of $\theta$, this implies that ILA is less strict about step size tuning than ULA. As will be seen in \S\ref{sec:Numerics}, this will prove to have a profound effect on the performance of ILA relative to ULA on high-dimensional problems.

\subsection{Asymptotic exactness for the normal distribution}

\label{sec:SecondOrder}
Among all values for $\theta$, we draw special attention to the choice $\theta = 1/2$. This resulting integrator, also known as the trapezoidal scheme, is known to be {\em second-order accurate} when applied to ODEs; that is, for a quadratic function $F$, iterates of the trapezoidal scheme for solving $y' = F(y)$ yield points of the {\em exact} solution \citep[\S12.4]{suli2003introduction}. An important consequence of this is that the global error incurred in the trapezoidal scheme is $\mathcal{O}(h^2)$ as $h \to 0$. Unfortunately, as a consequence of the It\^{o} calculus, this property does not hold for numerical solutions of stochastic differential equations. {{The construction of second-order schemes which are exact for quadratic $f$ in finite time and exhibit $\mathcal{O}(h^2)$ global error generally require careful treatment of the stochastic term---see for example \citet{anderson2009}, or the Ozaki local linearization scheme \citep{biscay1996local}. However, when $\theta = 1/2$, the notion of second-order accuracy itself carries over to ILA in a rather remarkable way.}}

The case where $f$ is a quadratic form corresponds to sampling from a (multivariate) normal distribution $\mathcal{N}(\v \mu, \v\Sigma)$, where
\begin{equation}
\label{eq:NormalExample}
f(\v x) = \tfrac12 (\v x - \v \mu)^{\top} \v\Sigma^{-1} (\v x - \v \mu).
\end{equation}
It is easy to see that, in this particular setting, \cref{eq:overdamped_Langevin_theta_method_0} becomes explicitly solvable. Indeed, letting $\m Q = \v\Sigma^{-1}$, we see that $\nabla f(\v x) = \m Q (\v x - \v \mu)$, and so
\begin{equation}
\label{eq:NormalExampleIters}
\v X_{k+1} = \left(\m I + \frac{h\theta}{2} \m Q\right)^{-1}\left[\left(\m I - \frac{h(1-\theta)}{2}\m Q\right) (\v X_k - \v \mu) + \sqrt{h} \v Z_k\right] + \v \mu.
\end{equation}
Observe that if $\v X_0$ is chosen to be a fixed value, all of the iterates $\v X_k$ are normally distributed. As a consequence of L\'{e}vy continuity, the stationary distribution of the ILA, if it exists, must also be normally distributed. In particular, due to (\ref{eq:NormalExampleIters}), it must have mean $\v m$ and covariance $\m V$ satisfying
\begin{align*}
\v m - \v \mu &= (\m I + \tfrac12 h \theta \m Q)^{-1}(\m I - \tfrac12 h(1-\theta) \m Q) (\v m - \v \mu)\\
\m V &= (\m I + \tfrac12 h \theta \m Q)^{-2}[(\m I - \tfrac12 h (1-\theta) \m Q)^2 \m V + h \m I].
\end{align*}
Since $\m Q \neq \m 0$, it must be the case that $\v m = \v \mu$. Solving for $\m V$,
the stationary distribution of the ILA is found to be
\[
\mathcal{N}\left(\v \mu, \v\Sigma\left(\m I + \tfrac12 h(\theta-\tfrac12) \m Q\right)^{-1}\right).
\]

Here we encounter the remarkable fact that when $f$ is quadratic and $\theta = 1/2$, regardless of the step size chosen, ILA is \emph{asymptotically unbiased!}
To our knowledge, this was first observed in \citet{wibisono2018sampling}, however, as a consequence of our analysis, we can now deduce that $\theta = 1/2$ is the {\em only} choice of $\theta$ that yields this property. While asymptotic exactness is unlikely to hold for other sampling problems, it suggests that cases involving approximately quadratic $f$ should see near optimal performance when $\theta = 1/2$.

\section{Inexact Implicit Langevin Algorithm (i-ILA)}
\label{sec:Proximal}

It is clear that the utility of ILA is dependent on the solvability of \cref{eq:overdamped_Langevin_theta_method_0}. Fortunately, this is made feasible
by a useful reinterpretation of solutions to (\ref{eq:overdamped_Langevin_theta_method_0}) as those of a corresponding optimization problem.
Indeed, the inverse operator $({\cal I} + \tfrac{1}{2}h\theta  \nabla f )^{-1}$
is quite commonly considered in convex optimization, as it is equivalent to the proximal operator $\prox_{\frac12 h \theta f}$ defined by
\[
\prox_{f}(\v v) = \argmin_{\v x \in \mathbb{R}^d} \left\{ f(\v x) + \frac12 \|\v x - \v v\|^2\right\},\qquad \v v \in \mathbb{R}^d.
\]
This equivalence follows from that of an optimization problem and the root-finding problem for its critical values~\citep[Eqn.\ (3.4)]{parikh2014proximal}. Therefore, (\ref{eq:overdamped_Langevin_theta_method_0}) can be formulated as the following optimization problem (after rescaling by $2 / h$):
\begin{subequations}
\label{eq:overdamped_Langevin_theta_method}
\begin{align}
\label{eq:Subproblem_opt}
\v X_{k+1} &= \argmin_{\v x \in \bb R^d} \; F(\v x;\v X_k,\v Z_k),
\end{align}
where
\begin{align}
\label{eq:Subproblem_F}
F(\v x;\v y,\v z) \triangleq \theta f(\v x) + \frac{1}{h} \left\lVert
\v x - \left(\v y - \frac{h(1-\theta)}{2} \nabla f(\v y) + \sqrt{h} \v z\right) \right\rVert^2.
\end{align}
\end{subequations}
This reinterpretation of (\ref{eq:overdamped_Langevin_theta_method_0}) was also noted in \citet{wibisono2018sampling}, although only the $\theta=1/2$ case was considered. Iterations of the form \eqref{eq:overdamped_Langevin_theta_method} are often referred to as \emph{proximal-point methods} in the optimization literature \citep{combettes2011proximal, parikh2014proximal}. We remark that proximal operators were used in~\citet{Pereyra2016} in the construction of a \emph{proximal unadjusted Langevin algorithm} (P-ULA). In fact, iterates of their P-ULA algorithm would correspond with (\ref{eq:overdamped_Langevin_theta_method}) when $\theta = 1$ and if the Gaussian term $\sqrt{h}\v Z_k$ were to be moved outside the proximal operator. 
As one might expect, this discrepancy has a significant impact on the covariance of each proposal as $h\to\infty$; it will be shown in Theorem~\ref{thm:GaussApprox} that the asymptotic covariance of these proposals is generally anisotropic.

The implementation of Algorithm \ref{alg:sila_exact} now hinges entirely on our ability to solve the subproblem \cref{eq:overdamped_Langevin_theta_method}. For the unadjusted Langevin algorithm where $\theta = 0$, this can be done trivially through a closed form solution.  However, for $ \theta > 0 $, we generally have to resort to an iterative optimization scheme to solve \cref{eq:overdamped_Langevin_theta_method}. 
Thus far, we have assumed that the optimization problem in \cref{eq:overdamped_Langevin_theta_method} can be solved exactly. 
However, more often than not this is infeasible, and one must instead consider the effects of \emph{approximate} solutions of \cref{eq:overdamped_Langevin_theta_method} in the overall convergence of the chain. This results in a sampling variant, which is henceforth referred to as i-ILA (for inexact ILA).

The most natural way of doing this is by measuring the error in the corresponding
root-finding problem (\ref{eq:overdamped_Langevin_theta_method_0}) via the norm
of the gradient of the subproblem \cref{eq:Subproblem_F}, $\|\nabla F\|$.
This is ideal because not only  can it be readily computed in practice, but also the termination criterion of many iterative optimization algorithms is based on this norm falling below a given tolerance; for example, see \citet{nocedal2006numerical}. 
Furthermore, efficient algorithms for directly minimizing $ \|\nabla F\| $, as a surrogate function for optimization of $ F $, have been recently proposed, which enjoy linear, that is, geometric, convergence rates, even in the absence of smoothness or convexity of $ F $ \citep{roosta2018newton}. In addition, for sampling in distributed computational environments, such as when large-scale data cannot be stored on a single machine, distributed variants of these surrogate optimization algorithms have also been recently considered \citep{crane2019dingo}. These algorithms are particularly suitable as part of i-ILA since they are guaranteed to rapidly and monotonically decrease $ \|\nabla F\| $; recall that $ \|\nabla F\| $ need not be monotonically decreasing in optimization algorithms that optimize $ F $ directly. 

With this in mind, we consider an inexact modification of Algorithm \ref{alg:sila_exact}, shown in Algorithm \ref{alg:sila_inexact}, for approximate sampling from $\pi$.

\begin{algorithm}[htbp]
	\SetAlgoLined
	\SetKwInOut{Input}{Input}\SetKwInOut{Output}{output}
	\Input{- Initial value $\v X_{0} = \v x_0 \in \mathbb{R}^d$  \\
		- Number of samples $n$ \\
		- Step size $h > 0$ \\
		- $\theta$-method parameter $\theta \in [0,1]$ \\
		- Sub-problem inexactness tolerance $\epsilon \geq 0$}

	\For{$k=0,1,\ldots,n$}{
		Draw $\v Z_k \sim \mathcal{N}(0,\m I)$ \\
		Find $\v X_{k+1}$ satisfying $\|\nabla F(\v X_{k+1};\v X_k, \v Z_k)\| \leq \epsilon$, where $F$ is defined in \eqref{eq:Subproblem_F}
	}
	\caption{Inexact Implicit Langevin Algorithm (i-ILA)}
	\label{alg:sila_inexact}
\end{algorithm}

\subsection{Theoretical analysis of \cref{alg:sila_inexact}}
\label{sec:UserFriendly}
The increased stability offered by \cref{alg:sila_exact} has been established in \cref{thm:GeoErgodic}. However, while \cref{thm:GeoErgodic} guarantees rapid convergence towards \emph{some} stationary distribution, closeness of the $\theta$-method iterates to the target distribution $\pi$ and the effect of increasing $h$ on its bias as a sampling method, has yet to be established. Furthermore, \cref{alg:sila_exact} and its guarantees given by \cref{thm:GeoErgodic} require exact solutions of the root-finding problem \cref{eq:overdamped_Langevin_theta_method_0}, whereas Algorithm \ref{alg:sila_inexact} allows for such problems to be solved only inexactly. To address both of these problems, we devote this section to the development of theoretical guarantees of \cref{alg:sila_inexact}, inspired by the techniques of~\citet{dalalyan2017theoretical}. 
These guarantees come in the form of rate of convergence estimates under the 2-Wasserstein metric, defined between two probability
measures $\nu$ and $\pi$ by
\begin{align*}
W_2(\nu, \pi) = \inf_{\v X \sim \nu, \v Y \sim \pi} \|\v X - \v Y\|_{L_2}
\end{align*}
where the infimum is taken over all couplings $(\v X,\v Y)$ of $\nu$ and $\pi$, and is 
attained by some {\em optimal} coupling~\citep[Thm.\ 4.1]{villani2008optimal}.
The 2-Wasserstein metric can be readily linked to other quantities
of interest. For example, from the Kantorovich-Rubinstein formula~\citep[Eqn. (5.11)]{villani2008optimal}, for any $M$-Lipschitz function $\varphi$, we have that
\begin{align*}
|\nu(\varphi) - \pi(\varphi)| \triangleq \left| \int \varphi \, \df (\nu - \pi) \right|  \leq M W_2(\nu,\pi).
\end{align*}

Our guarantees will require the same assumptions on $f$ seen in \citet{dalalyan2017theoretical}, that are Assumptions \ref{ass:Lipschitz} and \ref{ass:Convexity}. 
Under these assumptions, the condition number of $ F $ in \cref{eq:Subproblem_F} can be written as 
\begin{align}
	\label{eq:kappa}
	\kappa_h \triangleq \frac{1+ \frac{1}{2}\theta h M }{1 + \frac{1}{2}\theta h m}.
\end{align}
Recall that the condition number \cref{eq:kappa} encodes and summarizes the curvature (the degree of relative flatness and steepness), of the graph of $F$.   
In optimization, it is well-known that a large condition number typically amounts to a more difficult problem to solve, and hence algorithms that can take such contorted curvature into account (Newton-type methods, for example), are more appropriate \citep{roosta2018SSN,xuNonconvexTheoretical2017}. It is only natural to anticipate that challenges corresponding to problem ill-conditioning similarly carry over to sampling procedures as well. 
Indeed, large ratios of $ M/m $, which imply increasingly anisotropic level-sets for $ f $, can hint at more difficult sampling problems. For example, this difficulty directly manifest itself in ill-conditioning of $ F $, which in turn results in more challenging sub-problems. 
Furthermore, in such situations, taking a larger step size can only exacerbate the ill-condition of $ F $. 
As a result, similar to the role played by second-order methods in optimization, one can naturally expect to see implicit methods to be more appropriate for ill-conditioned sampling problems. 

Under Assumptions~\ref{ass:Lipschitz} and \ref{ass:Convexity}, the discrepancy between the inexact variant of the $\theta$-method given in \cref{alg:sila_inexact}  and the target density $\pi$ under the 2-Wasserstein metric is described in \cref{thm:MidGuarantee}. 
 
~
\begin{theorem}
\label{thm:MidGuarantee}
Suppose $f$ satisfies Assumptions \ref{ass:Lipschitz} and \ref{ass:Convexity}. Let $\theta \in (0,1]$ and 
let $\nu_t$ denote the distribution of the iterate $\v X_t$ obtained by Algorithm \ref{alg:sila_inexact}, for each $t \geq 1$, starting from $\v X_0 \sim \nu_0$. 
Let $\kappa_h$ be as in \eqref{eq:kappa}, and if $\theta < 1$, let
\begin{equation}
\label{eq:HSwitch}
h^{\ast} = \frac{(\theta-\frac12)(M+m) + \sqrt{(\theta-\frac12)^2(M+m)^2 + 4\theta(1-\theta)mM}}{\theta(1-\theta)m M}.
\end{equation}
Furthermore,
\begin{enumerate}[label=(\roman*)]
\item if $h \leq h^{\ast}$ or $\theta = 1$, then let
\begin{equation}
\label{eq:Rho1}
\rho = \frac{1 - \frac12 h(1-\theta)m}{1 + \frac12 h \theta m},\quad \mbox{and}\quad
C = \frac{\kappa_h}{{m}};
\end{equation}
\item alternatively, if $\theta < 1/2$ and $h^{\ast} < h < \frac{4}{M(1-2\theta)}$, or if
$1/2 \leq \theta < 1$ and $h > h^{\ast}$, then let
\begin{equation}
\label{eq:Rho2}
\rho = \frac{\frac12 h(1-\theta)M - 1}{\frac12 h\theta M + 1},\quad \mbox{and}\quad
C = \frac{\frac12 \kappa_h^2 h}{2 + \frac12 h (2\theta - 1)M}.
\end{equation}
\end{enumerate}
Then, for any $t \in \bb N$,
\begin{equation}
\label{eq:WasserDecay}
W_2(\nu_t, \pi) \leq \kappa_h \rho^t W_2(\nu_0,\pi) + C \left( \epsilon + \min\left\{{M}\sqrt{h d}(2+\sqrt{h M}) , {2} \sqrt{Md}\right\} \right).
\end{equation}
\end{theorem}

\begin{remark}
As $\theta \to 0$, for the transition point $h^{\ast}$ we have  $h^{\ast} \to \frac{4}{M+m}$. {Moreover, at $\theta = 0$ and $\epsilon = 0$, (\ref{eq:WasserDecay}) coincides with \citet[Theorem 1]{dalalyan2017user}, up to a different constant in the bias term. 
To see this, observe that $\theta = 0$ implies $\kappa_h = 1$, $\rho = 1 - \frac12 h m$, $C = m^{-1}$ when $h \leq \frac{4}{M + m}$, and $\rho = \frac12 h M -1$, $C = \frac{h}{4 - h M}$ when $\frac{4}{M + m} < h < \frac{4}{M}$. This gives
\[
W_2(\nu_t, \pi) \leq \begin{cases}
(1 - \tfrac12 h m)^t W_2(\nu_0, \pi) + \frac{4M}{m} \sqrt{h d} \\
(\tfrac12 h M - 1)^t W_2(\nu_0, \pi) + \frac{4 M h}{4 - h M} \sqrt{hd}.
\end{cases}
\]}
Theorem \ref{thm:MidGuarantee} may thus be seen as a generalization of \citet[Theorem 1]{dalalyan2017user} to arbitrary $\theta \in [0,1]$ and error $\epsilon$.
\end{remark}
{
\begin{remark}
In the noteworthy case of $\theta = 1/2$, Theorem \ref{thm:MidGuarantee} implies
\[
W_2(\nu_t, \pi) \leq \frac{1+\frac14 hM}{1 + \frac14 hm}\cdot\begin{cases}
\left(\frac{1-\frac14 hm}{1+\frac14 hm}\right)^t W_2(\nu_0, \pi) + \frac{\epsilon + M \sqrt{ h d}(2 + \sqrt{h M})}{m} & \mbox{ if } h \leq \frac{4}{\sqrt{m M}} \\
\left(\frac{\frac14 h M - 1}{\frac14 hM + 1}\right)^t W_2(\nu_0,\pi) + \frac12h\cdot \left(\frac{1 + \frac14 h M}{1 + \frac14 h m}\right)\left(\frac{\epsilon}{2} + \sqrt{M d}\right)  & \mbox{ otherwise}.
\end{cases}
\]
\end{remark}
}

As Theorem~\ref{thm:GeoErgodic} did for Algorithm~\ref{alg:sila_exact}, Theorem~\ref{thm:MidGuarantee} shows that for $\theta \geq 1/2$, Algorithm~\ref{alg:sila_inexact} is stable for all $h > 0$. Theorem~\ref{thm:MidGuarantee} does suggest that smaller values of $\theta$ will achieve faster convergence rates and smaller biases for small step sizes, although this does not appear to be the case in practice (for example, refer to \S\ref{sec:Numerics}).
Observe that, for $ \theta > 1/2 $ and fixed $ h $, the bias term is in the order of $ \mathcal{O}(M^{-1/2})$. This implies that increasing the condition number when $ m $ is bounded below (for example the spherical Gaussian prior in Bayesian regression) results in smaller bias and faster convergence. This is in sharp contrast to ULA whose performance significantly degrades with increasing condition number in such settings.

Also, we would like to reiterate that, in stark contrast to what is observed in \citet{roberts2001optimal} for Metropolis-Hastings algorithms, the rate of convergence in Theorem \ref{thm:MidGuarantee} for Algorithm \ref{alg:sila_inexact} is {\em not} dependent on
the dimension $d$ in any form other than through the appearances of $m$ and $M$. The dimension appears in the bias term simply due to the natural expansion of the Euclidean distance with dimension. In particular, following \citet{Durmus2016}, as the dependence on dimension is at most polynomial, this lends credence to the claim that implicit Langevin methods are well-equipped to handle high-dimensional sampling problems. 

\section{Asymptotics for large step size}
\label{sec:GaussApprox}
While Theorem~\ref{thm:MidGuarantee} provides an essential description of the behavior of \cref{alg:sila_inexact}, the bounds presented there are tightest for smaller step sizes on the order of $1/M$, and are less effective when $h$ is larger. 
Unfortunately, the most useful applications of ILA will occur when $M$ is large, and so the small step size ($h \to 0$) regime will not be all that relevant. Enabled by the increased stability of ILA, we present a novel analysis of Algorithm~\ref{alg:sila_exact} by establishing a central limit-type theorem regarding asymptotic behavior of the iterates in the $h \to \infty$ regime. 

Before we begin with a formal analysis, we are able to obtain insight by considering the behavior of the subproblem \cref{eq:overdamped_Langevin_theta_method} as $h \to \infty$. 
For $ h \gg 1 $, we have
\[
\frac1h \left\lVert \v x - \v x_t + \frac{h(1-\theta)}{2} \nabla f(\v x_t)
+ \sqrt{h} \v z_t \right\rVert^2
= \, (1-\theta) \v x \cdot \nabla f(\v x_t) + \mathcal{O}(h^{-1/2}) + C,
\]
where $C$ does not depend on $\v x$, and so does not contribute to solving \eqref{eq:overdamped_Langevin_theta_method}. As a result, the iterates of the $\theta$ method in the $h \to \infty$ regime will satisfy the relations $\v x_{t+1} = \v x_t^{\theta}$, where we let
\begin{equation}
\label{eq:h_large_iters}
\nabla f(\v x_{t}^{\theta}) = \left(1 - \frac{1}{\theta}\right) \nabla f(\v x_{t}),\qquad \v x \in \mathbb{R}^d.
\end{equation}
Letting $\v x^{\ast}$ denote the unique mode of $\pi$, iterating (\ref{eq:h_large_iters})
gives
\[
\|\nabla f(\v x_t) - \nabla f(\v x^{\ast})\| = \rho^t \|\nabla f(\v x_0) -\nabla f(\v x^{\ast})\|,
\]
where $\rho = \frac1{\theta} - 1$.
Under Assumption \ref{ass:Convexity}, we obtain
\[
\frac{m \rho^t}{M} \|\v x_0 - \v x^{\ast}\| \leq \|\v x_t - \v x^{\ast}\| \leq \frac{M \rho^t}{m} \|\v x_0 - \v x^{\ast}\|.
\]
Therefore, the behavior of
the $\theta$-method for large $h$ is determined according to the three regimes depicted in Table \ref{table:rho_regimes}. 
\begin{table}[htb]
\centering
\caption{Asymptotic behavior of iterates of (\ref{eq:overdamped_Langevin_theta_method}) as $h \to \infty$} \vspace{2mm}
\begin{tabular}{c |c | c}
\hline 
	$0 \leq \theta < 1/2$	 &    $\rho > 1$    &  $\|\v X_t\| \to \infty$ (unbounded in probability) \\
	$\theta = 1/2$	 &   $\rho = 1$     &  iterates oscillate about the mode  \\ 
	$1/2 < \theta \leq 1$	 &  $\rho < 1$      &  $\v X_t \to \v x^{\ast}$ (collapse to the mode) \\
	\hline                            
\end{tabular}\label{table:rho_regimes}
\end{table}
The $\theta < 1/2$ case is clearly undesirable from a practical standpoint. Moreover, the collapse towards
the mode seen when $\theta$ is close to one suggests enormous potential bias for large step sizes. On the other hand, the $\theta = 1/2$ case 
provides no damping effect whatsoever (a fact also supported by Theorem~\ref{thm:MidGuarantee}), making it susceptible to rare large proposals. Based on this preliminary analysis, for some small $\epsilon > 0$, a choice of $\theta = 1/2+\epsilon$ appears to provide the safest, and potentially the most accurate of our $\theta$-method samplers. This aligns with the rule-of-thumb used for $\theta$-method discretization of ODEs \citep[p.\ 85]{ascher2008numerical}. 
To formally extend these characterizations to the implicit $\theta$-method scheme \cref{eq:overdamped_Langevin_theta_method_0}, in Theorem \ref{thm:GaussApprox}, a central limit theorem as $h\to\infty$ is obtained for a single step of the scheme about
the deterministic map $\v x \mapsto \v x^{\theta}$. 
\begin{theorem}
\label{thm:GaussApprox}
Given any $f \in \mathcal{C}^2(\mathbb{R}^d)$, consider iterations given by \cref{eq:overdamped_Langevin_theta_method}, where $\theta \in (0,1]$. Conditioned on $ \v X_{k} $, we have
\begin{align*}
	\sqrt{h}(\v X_{k+1} - \v X_k^{\theta}) \underset{h \to \infty}{\overset{\mathcal{D}}{\longrightarrow}} \mathcal{N}\left(\v 0, \frac{4}{\theta^2} \nabla^2 f(\v X_k^{\theta})^{-2}\right). 
\end{align*}
\end{theorem}
Theorem \ref{thm:GaussApprox} implies that as $h\to \infty$, the implicit $\theta$-method scheme behaves similarly to a Markov chain $\{\v W_k\}$ with transitions
\[
\v W_{k+1} = \v W_k^{\theta} + \frac{2}{\theta\sqrt{h}} \nabla^2 f(\v W_k^{\theta})^{-1} \v Z_k,
\]
whose dynamics mimic those of the map $\v x \mapsto \v x^{\theta}$, but with an additional normally-distributed noise term at each step. Furthermore, the variance of this noise term increases as the implicit component of the scheme diminishes (taking $\theta \to 0$). {Despite the unusual $h \to \infty$ regime, we have found that, for typical step sizes, Theorem \ref{thm:GaussApprox} provides a surprisingly accurate description of the transition dynamics of ILA when $\theta \geq 1/2$, and is ideal for developing heuristics.}

\subsection{A heuristic choice for step size}
\label{sec:step_heuristic}
A consequence of the proof of Theorem \ref{thm:GaussApprox} is that the covariance $\v\Sigma_h(\v x)$ of the proposal density from the transition kernel $p(\v y \, \vert \, \v x)$ behaves asymptotically as
\[
\v\Sigma_h(\v x) \approx h \left(\m I + \frac{h\theta}{2} \nabla^2 f(\v x^{\theta})\right)^{-2},\qquad \mbox{as }h\to\infty.
\]
Conversely, {{applying the inverse function theorem to (\ref{eq:overdamped_Langevin_theta_inter}) reveals a linear approximation about $h = 0$, and hence}}
\[
\v\Sigma_h(\v x) \approx h \left(\m I + \frac{h\theta}{2} \nabla^2 f(\v x)\right)^{-2},\qquad \mbox{as }h \to 0.
\]
These two expressions coincide when $\v x = \v x^{\theta} = \v x^{\ast}$, where $ \v x^{\ast} $ denotes the mode. At this point, one might expect a `good' transition kernel to resemble the Laplace approximation of the distribution about $\v x^{\ast}$, which has covariance $\nabla^2 f(\v x^{\ast})^{-1}$. This suggests a heuristic for choosing a good step size in practice, by taking $h$ as a solution to the one-dimensional optimization problem
\begin{equation}
\label{eq:HeurProblem}
\hat{h}_{\theta} \coloneqq \argmin_{h \geq 0} \left\lVert h \left(\m I + \frac{h\theta}{2}\nabla^2 f(\v x^{\ast})\right)^{-2} - \nabla^2 f(\v x^\ast)^{-1}\right\rVert_E,
\end{equation}
where the norm $\|\cdot\|_E$ can be any matrix norm of choice.
Solutions to (\ref{eq:HeurProblem}) can be obtained using off-the-shelf methods in univariate optimization, such as golden section search \citep[\S9.5]{cottle2017linear}. 
We will show in the next section that, for several examples, the step size obtained from \cref{eq:HeurProblem} with Frobenius norm tends to be an effective choice in practice, especially for $\theta = 1/2$, where it reveals itself to be {\em near optimal} in all of our experiments. 
For this choice of norm, \cref{eq:HeurProblem} can be replaced by the equivalent problem
\begin{align}
\label{eq:HeurProblem2}
\hat{h}_{\theta} = \argmin_{h \geq 0} \sum_{k=1}^d \left[h\left(1 + \frac{h\theta}{2}\lambda_k\right)^{-2} - \frac{1}{\lambda_k}\right]^2,
\end{align} 
where $\lambda_1, \ldots, \lambda_d$ are the eigenvalues of $\nabla^2 f(\v x^{\ast})$.
One drawback is that solving \eqref{eq:HeurProblem} or \eqref{eq:HeurProblem2} require either inversion of $\nabla^2f(\v x^{\ast}$), or knowledge of its spectrum, respectively, both of which may be prohibitively expensive in high dimensions.
In many problems, however, it is reasonable to assume a certain distribution of its spectrum; for example, that the eigenvalues $\lambda_1\geq\cdots\geq\lambda_d$ of $\nabla^2 f(\v x^{\ast})$ decay exponentially:
\begin{equation}
\label{eq:SpectralDist}
\log \lambda_k \approx \left(1 - \frac{k-1}{d-1}\right)\log M + \frac{k-1}{d-1} \log m,\qquad k=1,\dots,d,
\end{equation}
where $m$ and $M$ take the place of the smallest and largest eigenvalues of $\nabla^2 f(\v x^{\ast})$, respectively. Simplifying assumptions such as these can be justified in problems where Hessian of $ f $ is approximately low rank, in the sense that it has a small stable rank \citep{roas1}, and hence its spectrum decays fast. Under these assumptions, solving \cref{eq:HeurProblem2} becomes more tractable; we shall make use of this for \cref{fig:MuskPlot} in \S\ref{sec:Numerics}.

\section{Numerical experiments}
\label{sec:Numerics}
In this section, we evaluate the empirical performance of Algorithms \ref{alg:sila_exact} and \ref{alg:sila_inexact} in  high-dimensions as measured by a few discrepancy measures. 
Recall that the total variation distance between any two absolutely continuous distributions with densities $p$ and $q$ over $\mathbb{R}^d$ respectively is given by
\begin{align*}
	\TV(p,q) \coloneqq \frac12 \int_{\mathbb{R}^d} |p(\v x)-q(\v x)|\dd \v x.
\end{align*}
Since the total variation metric is too difficult to directly estimate in higher dimensions, we follow the standard approach in the literature (see for example \citet{Durmus2016} and \citet{Maire2018}) and consider instead the \emph{mean marginal total variation} (MMTV), 
\[
\MMTV(p,q) \coloneqq \frac1{2d} \sum_{i=1}^d \int_{\mathbb{R}} |p_i(x)-q_i(x)|\dd  x,
\]
which we estimate as follows. First, kernel smoothing is applied to samples for each marginal from an extended MCMC run, as well as samples obtained from a single run of each method. The total variation between these {\em estimated} univariate densities is then computed with high accuracy via Gauss-Kronrod quadrature \citep{kahaner1989numerical}. 

As a weakness of MMTV is its inability to adequately compare the covariances within coordinates between the two sample sets, we also compare with a second discrepancy measure; \emph{maximum mean discrepancy} (MMD) \citep{gretton2012kernel,muandet2017kernel}. Letting $\mathcal{H}$ denote a reproducing kernel Hilbert space over $\mathbb{R}^d$ with reproducing kernel $k$, MMD is defined as the integral probability metric
\begin{align*}
\MMD^{2}(p,q) &\coloneqq \left[ \sup_{\|h\|_{\mathcal{H}} \leq 1} \int_{\mathbb{R}^d} h(x) [p(x)-q(x)] \dd x \right]^{2} \\ &= \mathbb{E}_{p,p} k(X,\tilde{X}) - 2 \mathbb{E}_{p,q} k(X,Y) + \mathbb{E}_{q,q} k(Y,\tilde{Y}),
\end{align*}
where $ X, \tilde{X} $ are independent random variables with distribution $ p $, and $ Y, \tilde{Y} $ are independent random variables with distribution $ q $.
These expectations can be estimated using samples from $p$ and $q$. In our experiments, we use the Gaussian kernel:
\[
k(\v x,\v y) = \exp\left(-\frac1{2\sigma^2} \|\v x - \v y\|^2\right),
\]
where the kernel bandwidth parameter $\sigma$ is chosen so that $2 \sigma^2$ is the median of $(\|\v x_i~-~\v x_j\|)_{i,j=1}^n$, where $\v x_i$ denotes the $i$-th sample taken from $q$.

\subsection{High-dimensional Gaussian distributions}
To highlight the effects of problem ill-conditioning, we once again consider 
sampling from a multivariate Gaussian distribution, as in (\ref{eq:NormalExample}), with
explicitly computable iterates given by (\ref{eq:NormalExampleIters}). It is easy to see
that $f$ satisfies Assumptions \ref{ass:Lipschitz} and \ref{ass:Convexity}. 
To show efficacy in higher dimensions, we will consider $d=1000$. 
Furthermore, to test the effects of ill-conditioning, we focus on three choices of $\v\Sigma$ with condition numbers $\kappa \in \{1, 100, 10^8\}$. Each $\v\Sigma$ is
generated using the method of \citet{bendel1978population} to uniformly sample a correlation matrix
with eigenvalues given by (\ref{eq:SpectralDist}) for $m = 1$ and $M = \kappa$. 
For simplicity, we take $\v \mu = \v 0$. 
MMTV and MMD discrepancies were computed between $\pi$ and 
samples of $N = 5000$ points
generated by Algorithm \ref{alg:sila_exact} with $\theta \in \{0,1/2,1\}$ and
a variety of step sizes $h$ (encompassing $4/M$ and the step size heuristics in
\S\ref{sec:GaussApprox}). Common random numbers were used, and no burn-in period was applied.
The results are shown in Figure \ref{fig:NormalPlots}. 

\begin{figure}[H]
\begin{tabular}{ccc}
\hphantom{$\kappa = 1$\quad} & {\bf Mean Marginal Total Variation} & {\bf Maximum Mean Discrepancy}
\end{tabular}
\begin{tabular}{cc}
$\kappa = 1\,$: &
\noindent\parbox[c]{\hsize}{\includegraphics[width=0.4\textwidth]{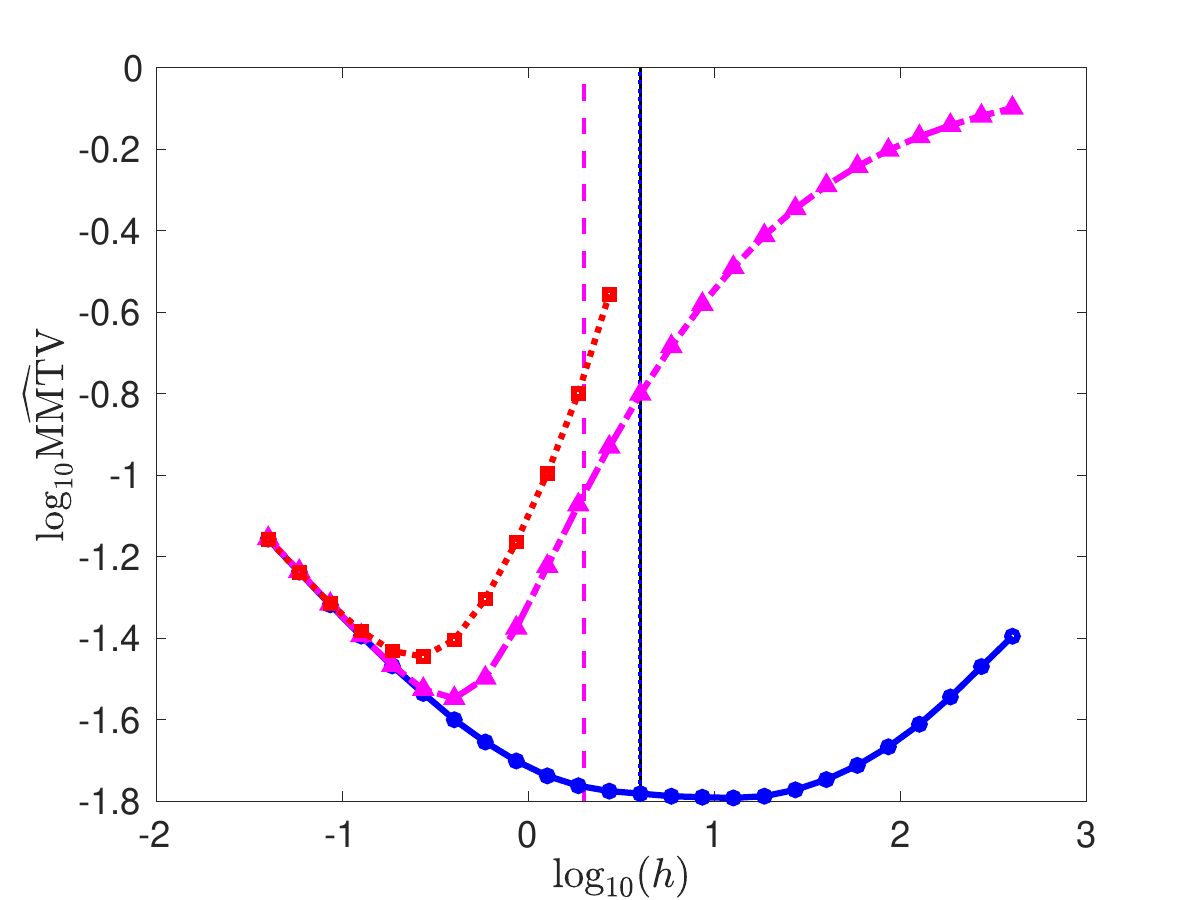}\includegraphics[width=0.4\textwidth]{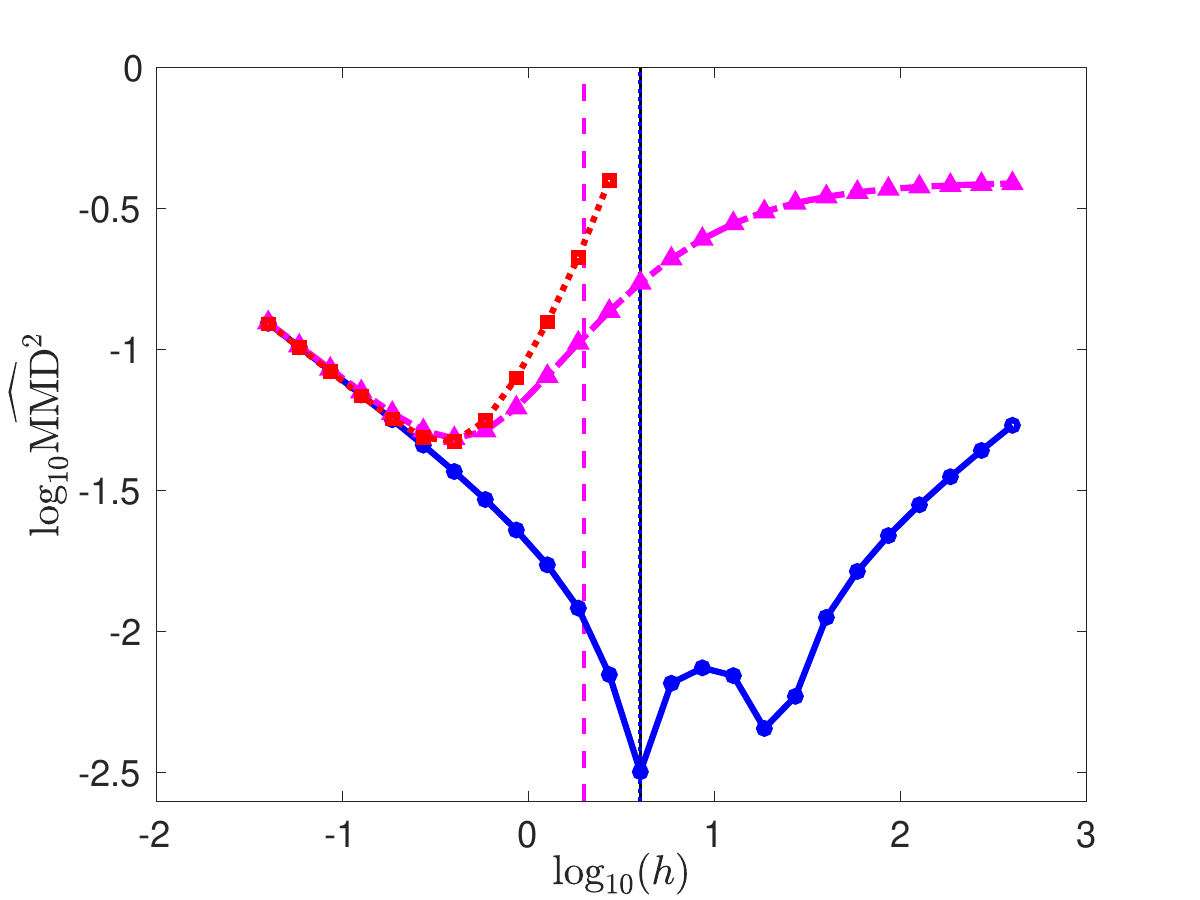}} \\
$\kappa = 100\,$: &
\noindent\parbox[c]{\hsize}{\includegraphics[width=0.4\textwidth]{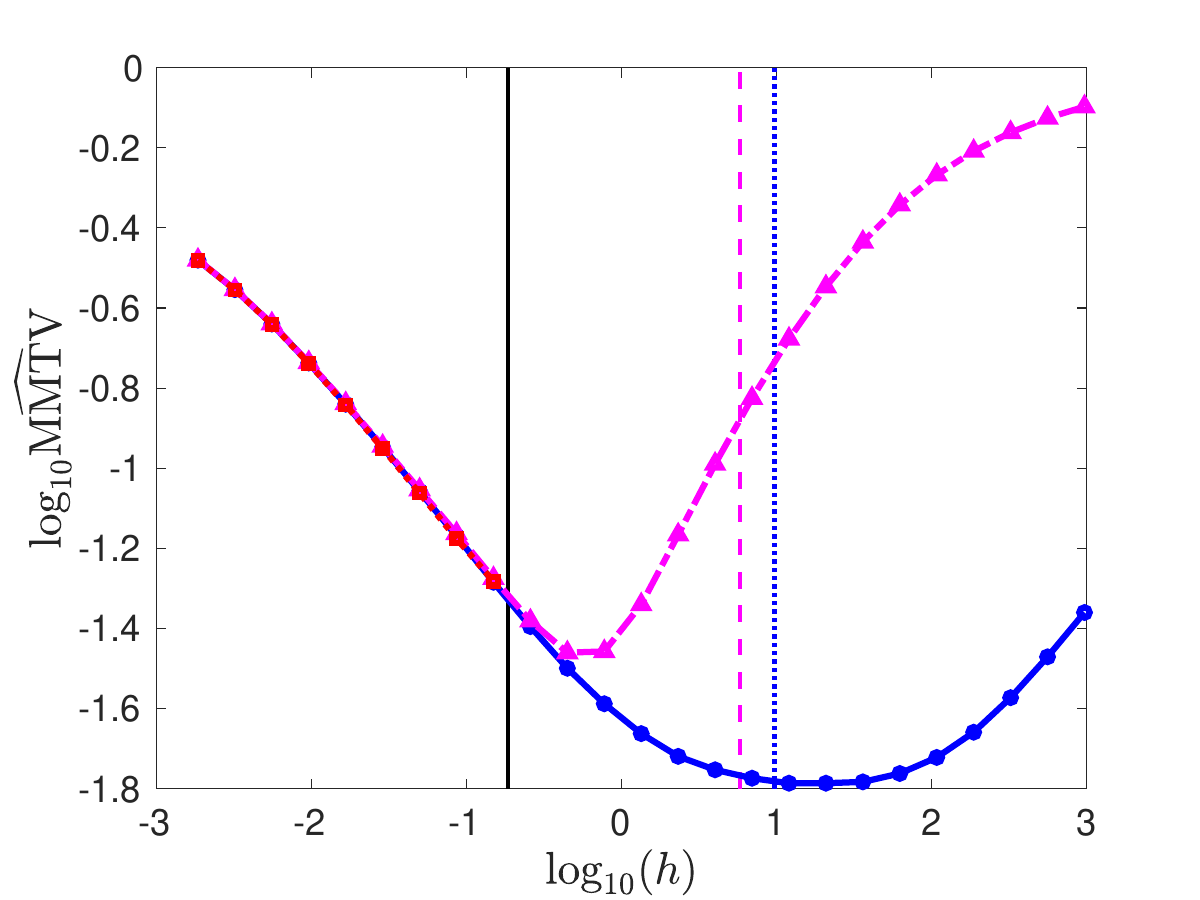}\includegraphics[width=0.4\textwidth]{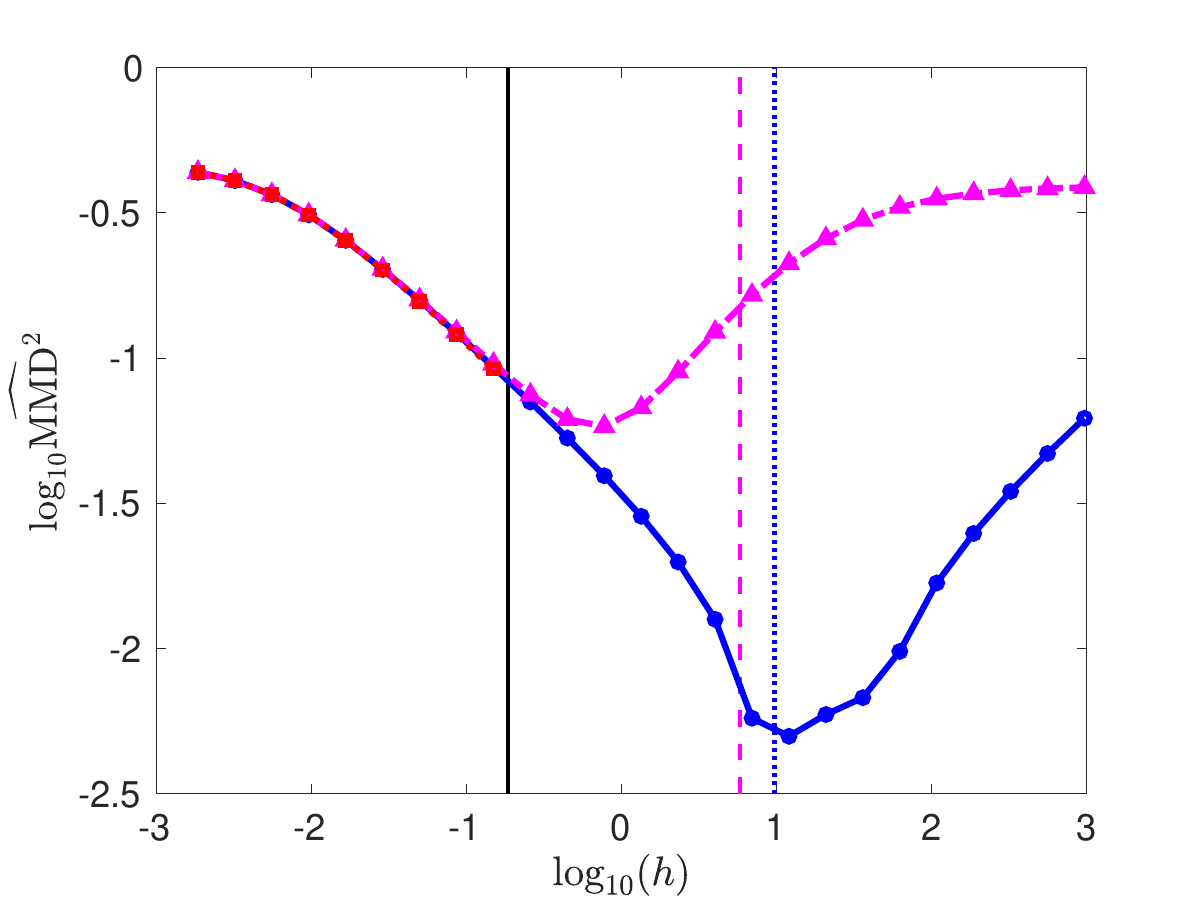}} \\
$\kappa = 10^8\,$: &
\noindent\parbox[c]{\hsize}{\includegraphics[width=0.4\textwidth]{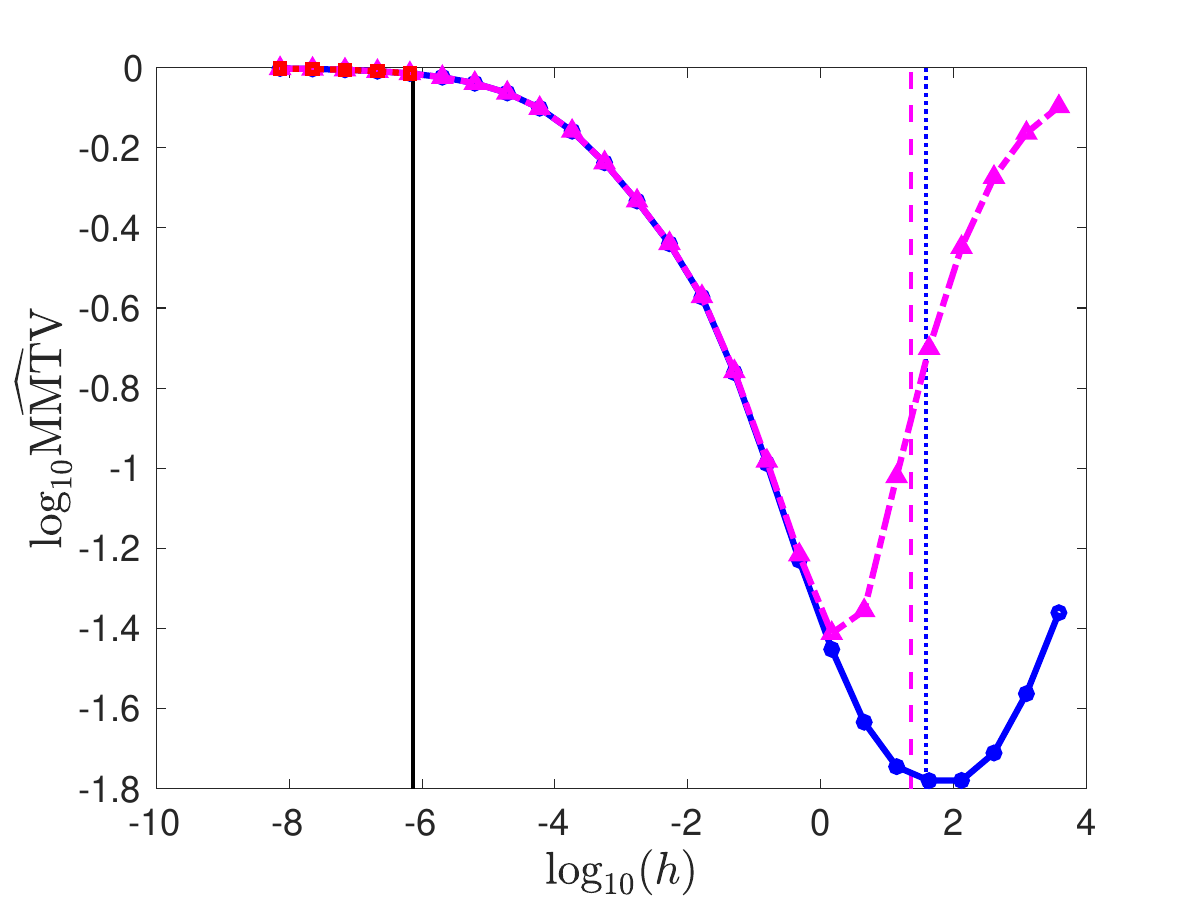}\includegraphics[width=0.4\textwidth]{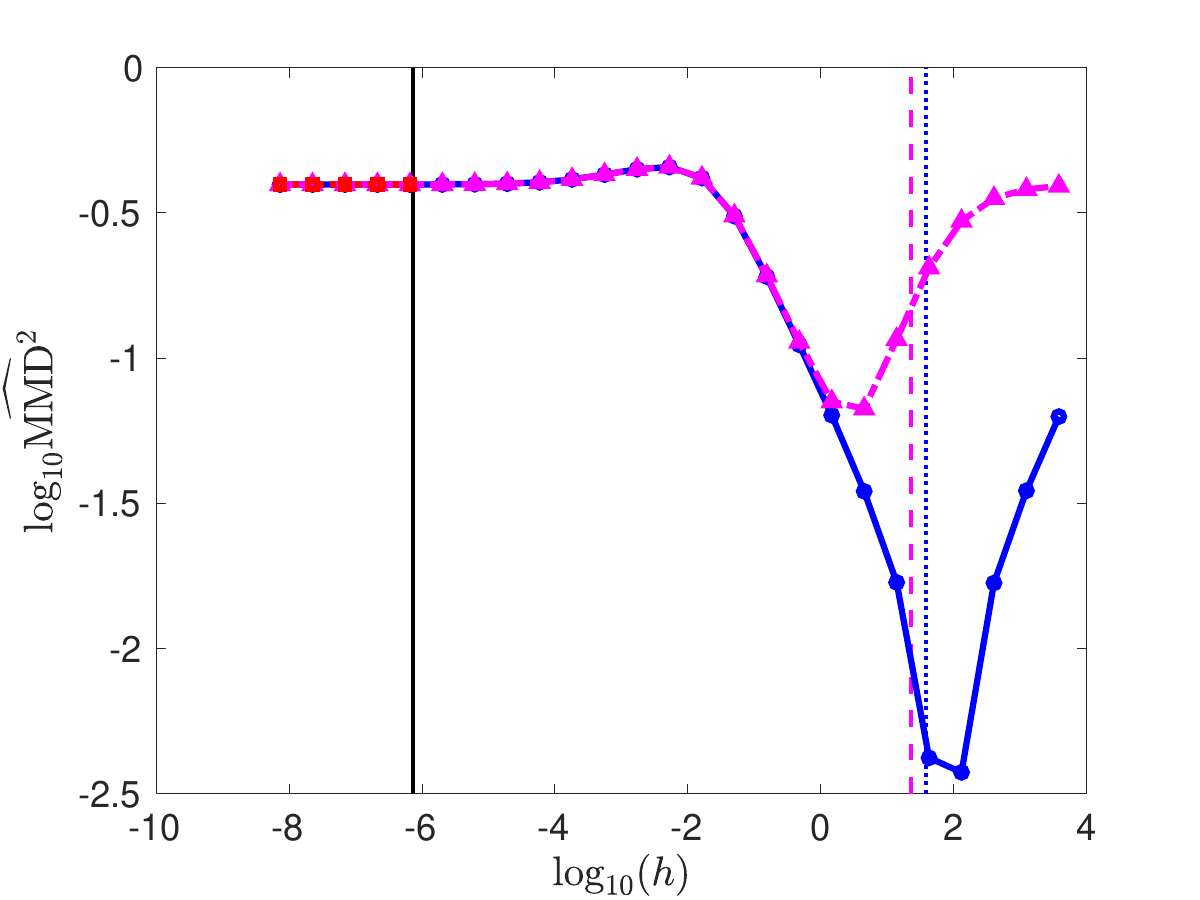}} \\
\end{tabular}
\begin{center}
\vspace{-0.25cm}
\includegraphics[width=0.4\textwidth]{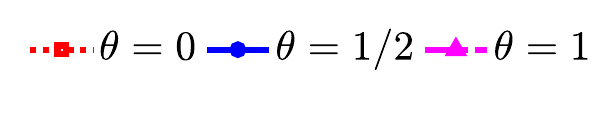}\\
\vspace{-0.5cm}
\includegraphics[width=0.5\textwidth]{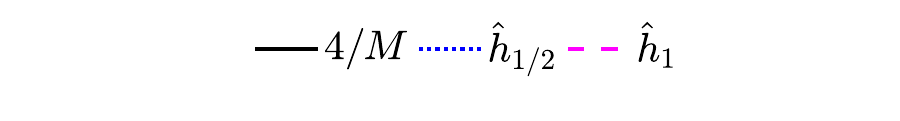}
\end{center}
\vspace{-.8cm}
\caption{MMTV estimates and MMD discrepancies for 5000 samples
generated by Algorithm \ref{alg:sila_exact} with $\theta \in \{0,1/2,1\}$, $h \in [\frac{4}{100M},100\hat{h}_{1/2}]$, and target distribution $\pi$ given by (\ref{eq:NormalExample}) with $\v \mu = \v 0$ and $\v\Sigma$ a correlation matrix with condition number $\kappa \in \{1,10^2,10^8\}$. \label{fig:NormalPlots}}

\end{figure}
Due to the rapid explosion in magnitude of samples generated by ULA when $h \geq 4/M$,
we only display discrepancies for ULA for $h < 4/M$. This critical value of
is highlighted as a black solid vertical line. 

In light of the large step size asymptotics, the existence of an ``optimal'' step size for $\theta\geq1/2$ as evidenced in these plots is perhaps not too surprising. However, it is surprising to see that, especially for large $\kappa$, this optimum is much greater than the maximum allowed step size of $4/M$ for ULA. Most notable here is the greatly improved performance of the implicit method ($\theta=1/2$) at this optimum over ULA for any allowable step size. Moreover, in all cases, the optimal performance of ILA for $\theta=1/2$ exceeds that of the purely implicit method ($\theta=1$). These two facts are not suggested by Theorem \ref{thm:MidGuarantee}, implying that the large step size asymptotics should indeed play a significant role in the analysis of implicit methods moving forward. 

In all cases, the step size heuristic for $\theta=1$ performs poorly, suggesting the fully implicit case operates by a different mechanism that is currently unknown to us. For $\kappa=1$, $\hat{h}_{1/2}~=~4/M~\equiv~4$, which is clearly the optimal step size, as it yields exact samples ($\v X_{k+1}=\v Z_{k}$). In fact, $\theta=1/2$, $h=4$ is the only choice of $\theta$ and $h$ which results in exact samples in this scenario. According to both estimated MMTV and MMD, the step size heuristic is an almost optimal choice of $ h $ for all $\kappa$ considered, even in high dimensions. 

\subsection{Logistic regression}
We now consider sampling problems involving the Bayesian posterior densities of generalized linear models (GLM),  which have log-concave likelihood functions, with Gaussian priors. For simplicity, and without loss of generality, we consider radially symmetric Gaussians. For a GLM with this choice of prior, posterior densities are {proportional to $\exp(-f(\v x))$, with} 
\begin{align*}
f(\v x) = \sum_{i=1}^{n} \left( \Phi(\v a_{i}^{\top} \v x) - b_{i} \v a_{i}^{\top} \v x \right) + \frac{\lambda}{2} \|\v x\|^{2},
\end{align*}
where $(\v a_{i},b_{i}), \; i = 1,2,\cdots,n,$ are the response and covariate pairs, $\v a_{i} \in \mathbb{R}^{p}$, and the domain of $b_{i}$ depends on the GLM. The cumulant generating function, $\Phi$, determines the type of GLM. For example, in the case of logistic regression, $ \Phi(t) = \log(1+\exp(t)) $; see~\citet{mccullagh1989generalized} for further details and applications. 
It is easy to see that
\begin{align*}
\nabla^2 f(\v x) = \sum_{i=1}^{n}\v a_{i} \v a^{\top}_{i} \Phi^{''}(\v a_{i}^\top \v x) + \lambda \m I = \m A^{\top} \m D \m A + \lambda \m I,
\end{align*} 
where $ \m A \in \reals^{n \times d}$ is a matrix whose $i$-th row is $ \v a_{i} $, $ \m D $ is a diagonal matrix whose $i$-th diagonal element is $ \Phi^{''}(\v a^\top \v x) $, and $\lambda$ is the precision parameter of the prior. As a result, for Assumption \ref{ass:Convexity}, we have
\begin{equation}
\label{eq:BayesRegSpecBounds}
\lambda \leq m \leq M \leq \|\m A\|^{2} \sup_{t \in \reals} \Phi^{''}(t) + \lambda. 
\end{equation}
For our example, we consider Bayesian logistic regression in this setting, yielding
\begin{equation}
\label{eq:MuskTarget}
f(\v x) \varpropto  \sum_{i=1}^{n} \left(\log\left(1 + \exp(\v a_{i}^{\top} \v x)\right) - b_{i} \v a_{i}^{\top} \v x \right) + \frac{\lambda}{2}||\v x||^2,
\end{equation}
and $\sup_{t \in \reals} \Phi^{''}(t) \leq 1/4$. We use the
\texttt{musk} (version 1) data set from the UCI repository \citep{Dua:2019}, with the prior precision parameter $\lambda = 1$. These choices yield a target distribution which is relatively
ill-conditioned, whose Hessian $\nabla^2 f(\v x^{\ast})$ about its mode $\v x^{\ast}$ possesses a condition number of $\kappa > 2 \times 10^3$. 
We estimate the values $m$ and $M$ according to their lower and upper bounds in (\ref{eq:BayesRegSpecBounds}). 
For the target density given according
to (\ref{eq:MuskTarget}), MMTV and MMD discrepancies were computed between samples
of $N = 10000$ points generated using Algorithm \ref{alg:sila_inexact} (with $\theta \in \{0,1/2,1\}$, $\epsilon = 10^{-9}$ and a variety of step sizes $h$ encompassing $4/M$ and the step size heuristics in \S\ref{sec:GaussApprox} under the assumption that eigenvalues are distributed according to Equation \ref{eq:SpectralDist}) and a gold standard run comprised of 50,000 samples obtained from hand-tuned SMMALA \citep{girolami2011riemann}. Due to the large difference in computation time between ULA and ILA, for the sake of comparison, we also computed MMTV and MMD discrepancies for the ULA algorithm using the same step sizes, now with a thinning factor of 50. This factor was chosen so that the computation time of ULA became roughly equivalent to the other ILA methods. Once again, common random numbers were used, and no burn-in period was applied. The results are shown in Figure \ref{fig:MuskPlot}, and follow a similar pattern to those found in the previous
example. The step size heuristic for $\theta = 1/2$ performs admirably in this case,
yielding samples with smaller discrepancies to the gold standard run than ULA for any reasonable step size, even when thinned to account for the difference in computation time.

\begin{figure}[t]
	\centering
\begin{tabular}{cc}
{\bf Mean Marginal TV} & {\bf Maximum Mean Discrepancy} \\
\includegraphics[width=0.4\textwidth]{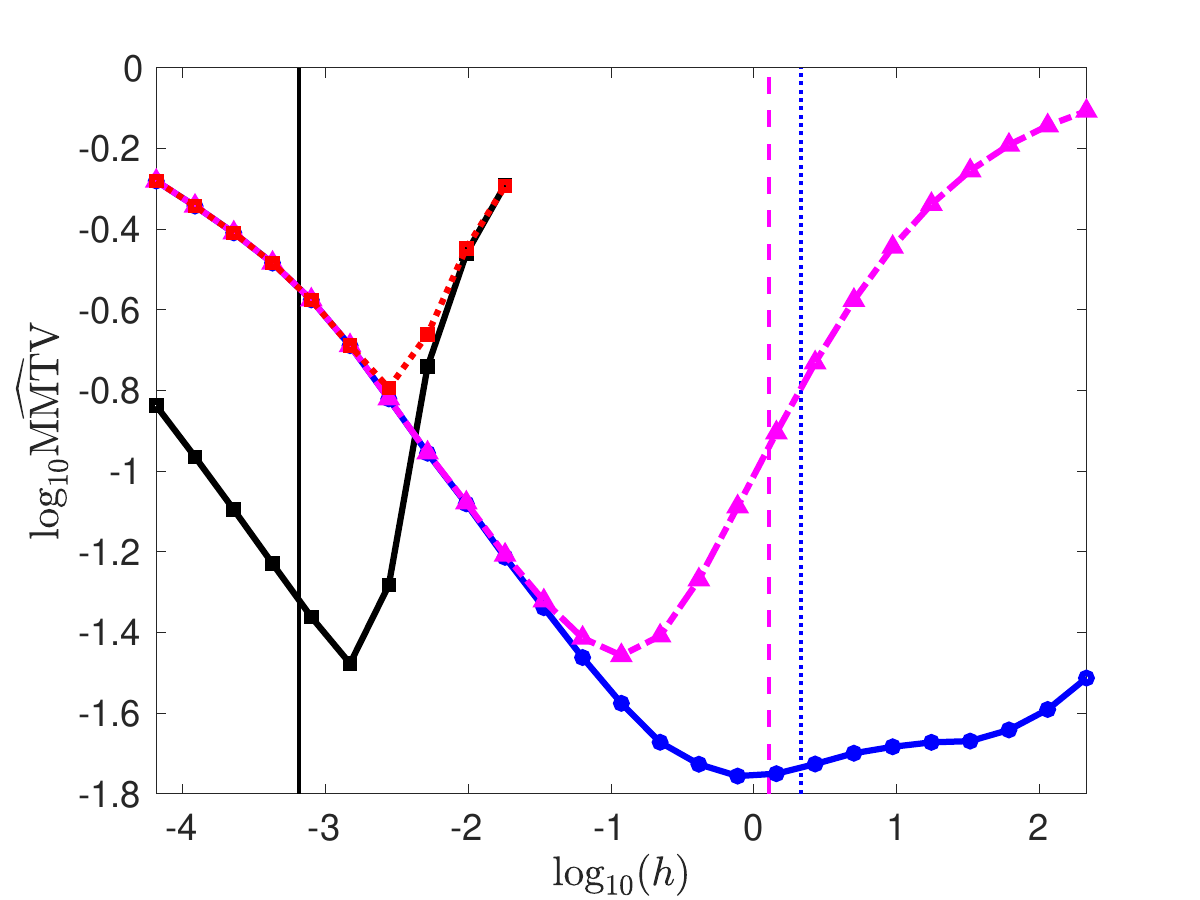} &
\includegraphics[width=0.4\textwidth]{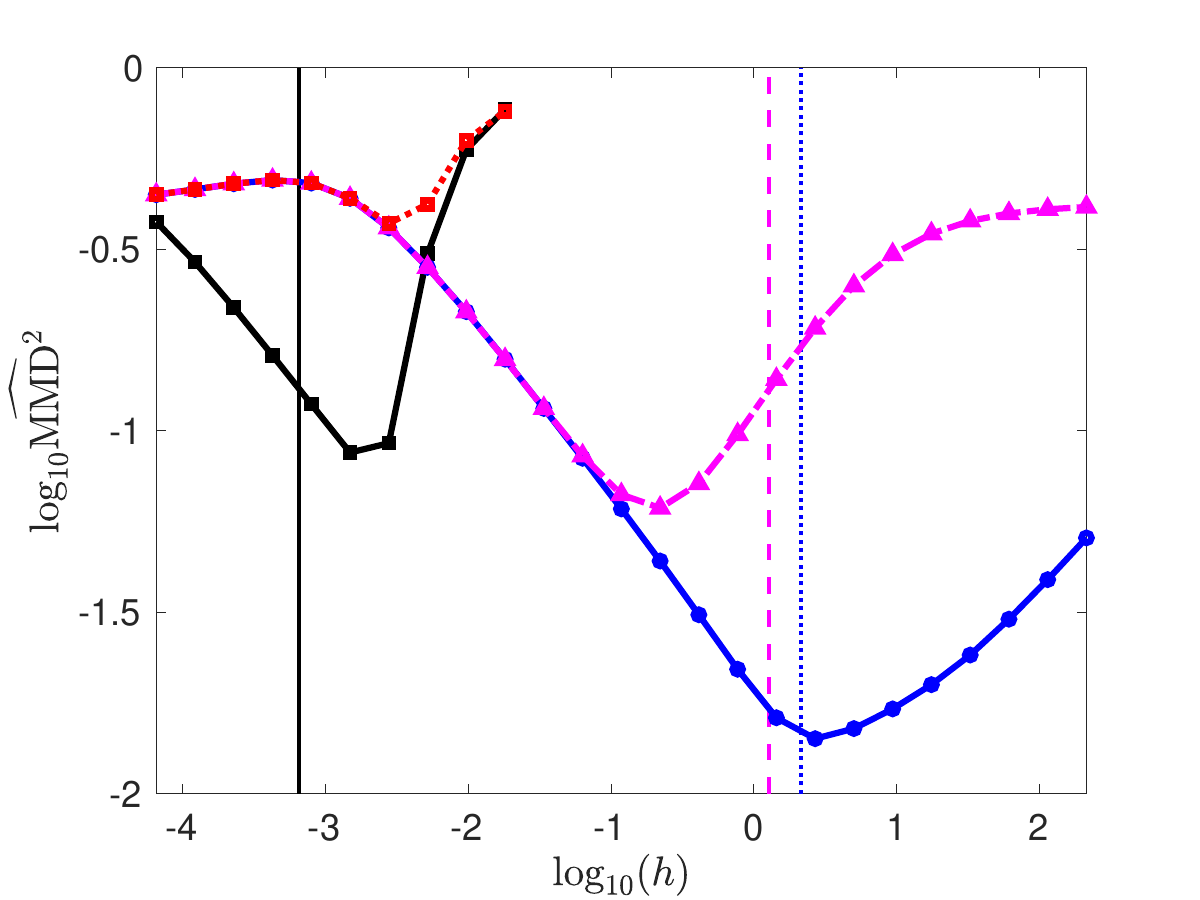}
\end{tabular}
\vspace{-0.25cm}
\begin{center}
\includegraphics[width=0.55\textwidth]{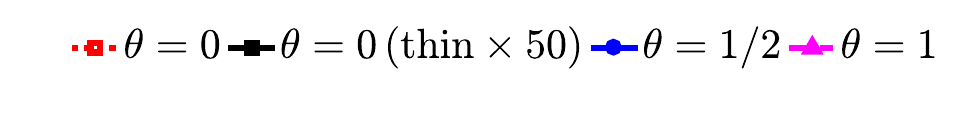}\vspace{-0.5cm}\\
\includegraphics[width=0.5\textwidth]{VerticalLegend.pdf}
\end{center}
\vspace{-0.5cm}
\caption{MMTV and MMD discrepancies between 10000 samples
generated by Algorithm \ref{alg:sila_inexact} with $\theta \in \{0,1/2,1\}$ and ULA with a thinning factor of 50, over $h \in [\frac{4}{10M},100 \hat{h}_{1/2}]$, and gold standard run, for target distribution specified according to (\ref{eq:MuskTarget}). \label{fig:MuskPlot}}
\end{figure}

\section{Conclusions}
\label{sec:conclusion}
In the context of sampling from an unnormalized probability distribution, we considered a general class of unadjusted sampling algorithms that are based on implicit discretization of the Langevin dynamics. 
Unlike the traditional Metropolis-adjusted sampling algorithms, these unadjusted methods relax the requirement of consistency that the sample empirical distribution should asymptotically be the same as the target distribution, and hence avoid incurring serious penalty to the mixing rate of the chain. As a result, these variants generate rapidly converging Markov chains whose stationary distributions are only approximations of the target distribution with a bias that is of adjustable size.
When one seeks a fixed (finite) number of samples, which is almost always the case in practice, this latter unadjusted view point can offer greatly many advantages. 

In this context, we focused on the class of discretization schemes 
generated using $\theta$-method in the context of smooth and strongly log-concave densities, explicitly deriving the transition kernel of the chain and establishing the corresponding sub-problems that are formulated as optimization problems. For smooth densities, the resulting implicit Langevin algorithms (ILA) have been shown to be geometrically ergodic for $ \theta \geq 1/2 $, irrespective of the step size. We also considered inexact variants (i-ILA) where the optimization sub-problems are solved only approximately. For this, we established non-asymptotic convergence of the sample empirical distribution to the target as measured by 2-Wasserstein metric, finding again that for $\theta > 1/2$, the resulting scheme is unconditionally stable for all step sizes. Furthermore, the growth rate in the bias term, that is shown to depend on problem's condition number, is greatly diminished for $\theta > 1/2$. Together with our numerical experiments, this suggests that the implicit methods are a more appropriate choice for ill-conditioned problems than explicit schemes. Furthermore, the case $\theta = 1/2$ appears to perform best in practice, especially when paired with our default heuristic choice of step size. The underlying reason for this is likely related to its asymptotic exactness for the normal distribution. It was suggested in
\citet{wibisono2018sampling} that the asymptotic bias of the $\theta=1/2$ case could be second-order accurate, which would imply the increased performance we have observed. Unfortunately, proving this claim remains an open problem.

Although enticing, extensions of these results to non-convex cases may
prove challenging due to the potential lack of unique solutions for the implicit scheme, and the relative difficulty of non-convex optimization in general. Nevertheless, one could find success in considering $f$ that is only strongly convex outside of a compact region, as in \citet{cheng2018sharp}. Furthermore, although it has not been treated explicitly, we believe that implicit methods should prove effective in big data problems, that is with $ f(\v x) = \sum_{i=1}^{n} f_{i}(\v x) $ and $ n \gg1 $, where it might be computationally prohibitive to evaluate $f$ or its gradient exactly. In this regard, one can use optimization algorithms that can employ inexact oracle information; see
\citet{roosta2018SSN} for example. The efficacy of this approach would prove interesting for future research. 

\section{Acknowledgements}
All authors have been supported by the Australian Centre of Excellence for Mathematical and Statistical Frontiers (ACEMS) under grant number CE140100049. Liam Hodgkinson acknowledges the support of an Australian Research Training Program (RTP) Scholarship. Fred Roosta was partially supported by the Australian Research Council through a Discovery Early Career Researcher Award (DE180100923). Part of this work was done while Fred Roosta  was visiting the Simons Institute for the Theory of Computing.

\clearpage
\appendix
\section{Proofs}
\label{sec:proof}

\subsection{Proof of Theorem 1 (Geometric Ergodicity)}
To establish geometric ergodicity, we prove the stronger Proposition \ref{prop:VUniform} below. 
First, we connect Assumptions \ref{ass:Lipschitz} and \ref{ass:SecantCond}
to the lower bound (\ref{eq:SecantCondConseq}).
\begin{lemma}
\label{lem:AssumpConnection}
The condition (\ref{eq:SecantCondConseq}) holds under Assumptions \ref{ass:Lipschitz} and \ref{ass:SecantCond}.
\end{lemma}
\begin{proof}
By \cref{ass:Lipschitz}, observe that for any $\v x, \v y \in \mathbb{R}^d$, we have
\begin{align*}
\frac{\langle \nabla f(\v x + \v y) - \nabla f(\v x) - \nabla f(\v y), \v x \rangle}{\|\v x\|^2} &{\boldsymbol{\le}} \frac{\langle \nabla f(\v x + \v y) - \nabla f(\v x), \v x \rangle}{\|\v x\|^2} - \frac{\langle \nabla f(\v y), \v x \rangle}{\|\v x\|^2} \\
&\leq \frac{\|\nabla f(\v x + \v y) - \nabla f(\v x)\|\|\v x\|}{\|\v x\|^2} + \frac{\|\nabla f(\v y)\|\|\v x\|}{\|\v x\|^2} \\
& = \frac{M \|\v y\| + \| \nabla f(\v y)\|}{\|\v x\|}.
\end{align*}
Hence, we have 
\begin{align*}
\liminf_{\|\v x\|\to\infty} \frac{\langle \nabla f(\v x) - \nabla f(\v y), \v x - \v y\rangle}{\|\v x - \v y\|^2} &= \liminf_{\|\v x\|\to\infty} \frac{\langle \nabla f(\v x), \v x \rangle}{\|\v x\|^2} \\
& \quad + \liminf_{\|\v x\|\to\infty} \frac{\langle \nabla f(\v x + \v y) - \nabla f(\v x) - \nabla f(\v y), \v x \rangle}{\|\v x\|^2} \\
&= \liminf_{\|\v x\|\to\infty} \frac{\langle \nabla f(\v x), \v x \rangle}{\|\v x\|^2} > 0.
\end{align*}
The result is implied by the definition of limit infimum.
\end{proof}
To state the result, we recall the definition of $V$-uniform ergodicity, as seen
in \citet{meyn2012markov}.
\begin{definition}[{$V$-uniform ergodicity}]
A $\nu$-ergodic Markov chain $\{\v X_n\}_{n=0}^{\infty}$ with Markov transition operator $\mathcal{P}$ on $\mathbb{R}^d$ is $V$-uniformly ergodic for a measurable function $V:\mathbb{R}^d\to[1,\infty)$ if
\[
\sup_{\v x \in \mathbb{R}^d} \sup_{|\phi| \leq V} \frac{|\mathcal{P}^n\phi(\v x) - \pi(\phi)|}{V(\v x)} \to 0,\qquad \mbox{as }n\to\infty.
\]
\end{definition}
By \citet[Theorem 16.0.1]{meyn2012markov}, any $V$-uniformly ergodic Markov chain is also geometrically ergodic. 
\begin{proposition}
\label{prop:VUniform}
For any $s > 0$ and $f$ satisfying Assumptions \ref{ass:Lipschitz} and \ref{ass:SecantCond}, let $V_s(\v x)$ denote the Lyapunov drift function
\begin{equation}
\label{eq:LyapunovDrift}
V_s(\v x) = \exp\left(s \| \v x - \v x^{\ast} + \tfrac12 h \theta \nabla f(\v x)\|\right),
\end{equation}
where $\v x^{\ast}$ is a critical point of $f$. Supposing that (\ref{eq:overdamped_Langevin_cond_h}) holds, the $\theta$-method scheme with
transition kernel (\ref{eq:overdamped_kernel}) is $V_s$-uniformly ergodic provided 
$\theta \geq 1/2$, or $\theta < 1/2$ and
\[h < \frac{4m}{M^2(1-2\theta)}.\]
\end{proposition}

\begin{proof}
It is immediately apparent from the positivity of (\ref{eq:overdamped_kernel}) due to
(\ref{eq:overdamped_Langevin_cond_h}) that the iterates of the $\theta$-method scheme
are aperiodic and irreducible with respect to Lebesgue measure. Furthermore, it follows from
\citet[Proposition 6.2.8]{meyn2012markov} that all compact sets are small.
Therefore, by \citet[Theorem 15.0.1]{meyn2012markov} and \citet[Lemma 15.2.8]{meyn2012markov}, it suffices to show that
\[
\limsup_{\|\v x\|\to\infty} \frac{\mathcal{P} V_s(\v x)}{V_s(\v x)} = 0,
\]
where $\mathcal{P}$ is the Markov transition operator of the $\theta$-method scheme. 
Indeed, by the definition of $\limsup$, for a given $0 < \lambda < 1$, there exists a $K > 0$ such that
\[
\sup_{\|\v x\|\geq K} \frac{\mathcal{P} V_s(\v x)}{V_s(\v x)} \leq \lambda,
\]
and so $\mathcal{P} V_s(\v x) \leq \lambda V_s(\v x) + \sup_{\|\v x\|\leq K} \mathcal{P} V_s(\v x)$ for any $\v x \in \mathbb{R}^d$. 
Letting $\v X_1$ denote the first step of the $\theta$-method scheme starting from $\v X_0 = \v x$, (\ref{eq:overdamped_Langevin_theta_inter}) implies
\[
\v X_1 + \tfrac12 h \theta \nabla f(\v X_1) \sim \mathcal{N}\left(\v x - \frac{h(1-\theta)}{2} \nabla f(\v x), \ h \m I\right).
\]
Thus, by letting $\v Z \sim \mathcal{N}(\v 0, \m I)$, we obtain $\mathcal{P} V_s(\v x) = \mathbb{E} \exp(s g(\v Z))$, where
\[
g(\v z) = \| \v x - \v x^{\ast} - \tfrac12 h(1-\theta) \nabla f(\v x) + \sqrt{h} \v z\|.
\]
By the reverse triangle inequality, $|g(\v z_1) - g(\v z_2)| \leq \sqrt{h}\|\v z_1 - \v z_2\|$ for any $\v z_1, \v z_2 \in \mathbb{R}^d$, and hence, $ g $ is $\sqrt{h}$-Lipschitz in $ \v z $. Consequently, we can apply the Gaussian concentration inequality \citep[Theorem 5.5]{boucheron2013concentration} to reveal
\[
\mathbb{E} e^{s g(\v Z)} \leq \exp\left(s \mathbb{E}g(\v Z) + \frac{h s^2}{2}\right).
\]
Since by Jensen's inequality,
\[
\mathbb{E} g(\v Z) \leq \sqrt{h d} + \|(\v x - \v x^{\ast}) - \tfrac12 h(1-\theta) \nabla f(\v x)\|.
\]
It follows that
\[
\frac{\mathcal{P}V_s (\v x)}{V_s(\v x)} \leq \exp ( s\sqrt{hd} + \tfrac12 h s^2 + s[T_1(\v x) - T_2(\v x)]),
\]
where
\[
T_1(\v x) = \|\v x - \v x^{\ast} - \tfrac12 h(1-\theta) \nabla f(\v x)\|\quad \mbox{and}\quad
T_2(\v x) = \|\v x - \v x^{\ast} + \tfrac12 h\theta \nabla f(\v x)\|.
\]
Therefore, if we can show that $T_1(\v x) - T_2(\v x) \to -\infty$ as $\|\v x\|\to \infty$, then the result follows. 
Since $T_1(\v x) - T_2(\v x) = (T_1(\v x)^2 - T_2(\v x)^2) / (T_1(\v x) + T_2(\v x))$, we may focus on the difference of the squares:
\begin{align}
T_1(\v x)^2 - T_2(\v x)^2&= \|(\v x-\v x^{\ast})-\tfrac{1}{2} h(1-\theta)\nabla f(\v x)\|^{2}-\|(\v x-\v x^{\ast})+\tfrac{1}{2}h\theta\nabla f(\v x)\|^{2} \nonumber\\
&=\|\v x-\v x^{\ast}\|^{2}-h(1-\theta)\dotprod{\nabla f(\v x),\v x-\v x^{\ast}}
+\tfrac{1}{4}h^{2}(1-\theta)^{2}\|\nabla f(\v x)\|^{2} \nonumber\\
&\qquad-\|\v x-\v x^{\ast}\|^{2}-h\theta\dotprod{\nabla f(\v x),\v x-\v x^{\ast}}-\tfrac{1}{4}h^{2}\theta^{2}\|\nabla f(\v x)\|^{2} \nonumber\\
&=-h\dotprod{\nabla f(\v x),\v x-\v x^{\ast}} +\tfrac{1}{4}h^{2}(1-2\theta)\| \nabla f(\v x)\|^{2}\nonumber\\
&\leq-hm\|\v x-\v x^{\ast}\|^{2}+c(\v x^{\ast})+\tfrac{1}{4}h^{2}\max\{0,1-2\theta\}M^{2}\| \v x-\v x^{\ast}\|^{2},\label{eq:SquareDiff}
\end{align}
where the last inequality follows from (\ref{eq:SecantCondConseq}).
Provided that $\frac12 h (1-2\theta) M^2 < m$ or $\theta \geq 1/2$, (\ref{eq:SquareDiff}) will be negative for sufficiently large $\v x$. 
Since also
\[
T_1(\v x) + T_2(\v x) \leq 2\|\v x - \v x^{\ast}\| + \tfrac12 h \|\nabla f(\v x)\|
\leq (2 + \tfrac12 h M) \|\v x - \v x^{\ast}\|,
\]
for any $\epsilon > 0$ and sufficiently large $\v x$,
\[
T_1(\v x) - T_2(\v x) \leq \frac{-hm+\frac14 h^2 \max\{0,1-2\theta\}M^2}{2+\frac12 hM} \|\v x - \v x^{\ast}\| + \epsilon,
\]
which implies the difference $T_1(\v x) - T_2(\v x) \to -\infty$ as $\|\v x\| \to \infty$, as required.
\end{proof}

\subsection{Proof of Theorem 2 ($W_2$ bounds)}
Next, using techniques analogous to those of \citet{dalalyan2017user},
we prove Theorem \ref{thm:MidGuarantee}. For the sake of brevity, we let $a \wedge b$ denote the minimum of any two quantities $a$ and $b$.
The following estimate is fundamental to the argument.

\begin{lemma}
\label{lem:BiasBound1}
Let $\v L_{t}$ be the solution to the (overdamped) Langevin equation
\[
\dd\v L_{t}=-\tfrac{1}{2}\nabla f\left(\v L_{t}\right)\dd t+\dd\v W_{t}
\]
for $f\in\mathcal{C}^{1}(\mathbb{R}^{d})$
such that $\nabla f$ is $M$-Lipschitz continuous. Then for any $h>0$, if $\v L_{0}\sim\pi$, 
\[
\left\lVert \int_{0}^{h}\nabla f(\v L_{t})-\nabla f(\v L_{0})\dd t\right\rVert _{L^2}\leq \frac12 h [M\sqrt{hd}(2 +\sqrt{hM})\wedge4\sqrt{Md}].
\]
\end{lemma}
\begin{proof}
Since $\v L_t$ is stationary, for any $t \geq 0$, $\|\nabla f(\v L_t)\|_{L^2} \leq \sqrt{Md}$ by \citet[Lemma 2]{dalalyan2017further}. Therefore, $\left\lVert \int_0^h \nabla f(\v L_t) - \nabla f(\v L_0) \dd t\right\rVert_{L^2} \leq 2 h \sqrt{M d}$. Furthermore, following the same procedure as in \citet[Lemma 4]{dalalyan2017user}
\[
\left\lVert \int_0^h \nabla f(\v L_t) - \nabla f(\v L_0) \dd t\right\rVert_{L^2} \leq \frac14 h^2 M^{3/2} d^{1/2} + \frac23 h^{3/2} M d^{1/2}.
\]
\end{proof}
With Lemma \ref{lem:BiasBound1} in hand, we may proceed with the proof of the main result.

~

\begin{proof}[Theorem \ref{thm:MidGuarantee}]
Letting $\v W_t$ denote a $d$-dimensional standard Brownian
motion independent of $\v X_k$ and $\v L_0 \sim \pi$, define the stochastic process $\v L$ by
\[
\v L_t = \v L_0 - \frac12 \int_0^t \nabla f(\v L_s) \dd s + \v W_t, \qquad \qquad t \geq 0.
\]
Evidently, $\v L_t$ is a realization of (\ref{eq:overdamped_Langevin}) and so is a reversible Markov process with
$\v L_t \sim \pi$ for every $t \geq 0$. Now, couple the inexact $\theta$-method
scheme $\v X_k$ satisfying
\[
\v X_{k+1} = \v X_k - \frac{h}{2}[\theta \nabla f(\v X_{k+1}) + (1-\theta) \nabla f(\v X_k)] + \sqrt{h}\v Z_k + \v E_k,
\]
for an appropriate error term $\v E_k$, to $\v L_t$, by letting $\v Z_k = h^{-1/2}[\v W_{(k+1)h} - \v W_{kh}]$ for each $k\geq 1$, and choosing $\v L_0$ such that $W_2(\pi_0, \pi) = \|\v X_0 - \v L_0\|_{L^2}$. Observing that, for any $k \geq 1$,
\begin{align*}
\v E_k &= \frac{h}{2} \theta \nabla f(\v X_{k+1}) + \v X_{k+1} - \v X_k + \frac{h}{2} (1-\theta) \nabla f(\v X_k) - \sqrt{h} \v Z_k \\
&= \frac{h}{2} \nabla F(\v X_{k+1}; \v X_k, \v Z_k),
\end{align*}
by construction, $\|\v E_k\|_{L^2} \leq \frac12 h \epsilon$.
For each $k$, let $\v D_k = \v L_{kh} - \v X_k$, 
observing that $W_2(\pi_0,\pi) = \|\v D_0\|_2$ and $W_2(\pi_k,\pi) \leq \|\v D_k\|_2$. Choosing some $k \geq 1$, for the sake of brevity,
we denote $\v L_t^{(k)} = \v L_{kh + t}$, which now satisfies
\begin{align*}
\v L_h^{(k)} = \v L_{kh + h} &= \v L_{kh} - \frac12 \int_0^h \nabla f(\v L_{kh+s})\dd s + \v W_{kh + h} - \v W_{kh} \\
&= \v L_0^{(k)} - \frac12 \int_0^h \nabla f(\v L_s^{(k)}) \dd s + \sqrt{h} \v Z_k.
\end{align*}
Altogether, we have
\[
\v D_{k+1} = \v D_k - \tfrac12 h[(1-\theta)\v U_k+\theta \tilde{\v U}_k]-
[(1-\theta)\v V_k + \theta \tilde{\v V}_k] + \v E_k,
\]
where
\begin{align*}
\v U_k &= \nabla f(\v X_k + \v D_k) - \nabla f(\v X_k) & & & \v V_k &= \frac{1}{2} \int_0^h \nabla f(\v L_s^{(k)}) - \nabla f(\v L_0^{(k)}) \dd s \\
\v \tilde{\v U}_k &= \nabla f(\v X_{k+1} + \v D_{k+1}) - \nabla f(\v X_{k+1}) & & &
\v \tilde{\v V}_k &= \frac{1}{2} \int_0^h \nabla f(\v L_{h-s}^{(k)}) - \nabla f(\v L_h^{(k)}) \dd s.
\end{align*}
An application of the fundamental theorem of calculus implies $\v U_k = \m F_k \v D_k$ and $\v \tilde{\v U}_k = \m F_{k+1} \v D_{k+1}$, where
\[
\m F_k = \int_0^1 \nabla^2 f(\v X_k + t \v D_k) \dd t.
\]
Altogether, $\v D_{k+1} = \m S_k \v D_k + \v T_k$ where $\v T_k = -(\m I + \frac{h\theta}{2}\m F_{k+1})^{-1}[(1-\theta)\v V_k + \theta \tilde{\v V}_k - \v E_k]$ and
\[
\m S_k = \left(\m I + \frac{h\theta}{2} \m F_{k+1}\right)^{-1}
\left(\m I - \frac{h(1-\theta)}{2} \m F_k \right).
\]
It can be verified using induction that the solution to this 
first-order non-homogeneous recurrence relation is given by
\[
\v D_{k} = \m S_{k-1}\cdots \m S_0 \v D_0 + \sum_{l=0}^{k-1} \m S_{k-1} \cdots \m S_{l+1} \v T_l.
\]
Now, observe that by denoting $G(\m X) = (\m I - \frac{h(1-\theta)}{2}\m X)(\m I+\frac{h\theta}{2}\m X)^{-1}$, for any $l < k$,
\begin{equation}
\label{eq:SjProd}
\m S_{k-1} \cdots \m S_l = \left(\m I + \frac{h\theta}{2} \m F_{k}\right)^{-1}
G(\m F_{k-1})\cdots G(\m F_{l+1}) \left(\m I - \frac{h(1-\theta)}{2}\m F_l\right).
\end{equation}
Since the eigenvalues of $\nabla^2 f$ are bounded above by $M$ and below by
$m$, so too are the eigenvalues of $\m F_k$ for each $k$. Therefore,
\begin{align}
\|G(\m F_k)\|_2 &= \max_{z \in [m,M]} \left| \frac{1 - \frac12 h(1-\theta)z}{1 + \frac12 h\theta z} \right|\nonumber \\& = \max\left\{\frac{1-\frac12 h(1-\theta)m}{1+\frac12 h\theta m},\frac{\frac12 h(1-\theta)M - 1}{\frac12 h\theta M + 1}\right\} \eqqcolon \rho. \label{eq:RhoDefn}
\end{align}
The transition between these regimes occurs at the point $h^{\ast}$ which is
the solution to
\[
\frac{1-\frac12 h(1-\theta) m}{1 + \frac12 h\theta m}
= \frac{\frac12 h(1-\theta) M - 1}{\frac12 h\theta M + 1}
\]
over $h > 0$. Equivalently, it is the solution to
\begin{multline*}
\tfrac{1}{2}h\theta M+1-\tfrac{1}{4}h^{2}\theta(1-\theta)mM-\tfrac{1}{2}h(1-\theta)m \\=\tfrac{1}{2}h(1-\theta)M+\tfrac{1}{4}h^{2}\theta(1-\theta)mM-1-\tfrac{1}{2}h\theta m,
\end{multline*}
and therefore to the quadratic equation
\[
\tfrac12 h(1-2\theta)(m+M) + \tfrac12 h^2 \theta (1 - \theta) m M - 2 = 0.
\]
It may be readily verified that $h^{\ast}$ as defined in (\ref{eq:HSwitch}) is the only positive solution. Furthermore, $\rho < 1$ provided that $\theta \geq 1/2$ or $h < 4 / [M(1-2\theta)]$.
Also, for any $j,k \geq 1$,
\begin{multline}
\label{eq:SjSingle}
\| (\m I + \tfrac12 h \theta \m F_j)^{-1}\|_2 \|\m I - \tfrac12 h(1-\theta) \m F_k\|_2 \\
\leq \frac{\max\{ 1 - \tfrac12 hm(1-\theta), \tfrac12 h M(1-\theta) - 1\}}{1+\frac12 hm\theta} \leq \kappa_h\rho,
\end{multline}
which further implies $\|\m S_j\|_2 \leq \kappa_h \rho$ for any $j$. 
Now combining (\ref{eq:SjProd}), (\ref{eq:RhoDefn}), and (\ref{eq:SjSingle}), for $k > l$, 
$\|\prod_{j=l}^{k-1} \m S_j\|_2 \leq \kappa_h \rho^{k-l}$,
and hence, altogether,
\[
\|\v D_k\|_2 \leq \kappa_h \rho^{k} \|\v D_0\|_2 + \sum_{l=0}^{k-1} \kappa_h\rho^{k-l}\|\v T_l\|_2. 
\]
Since $\v L_t$ is reversible and stationary, $\|\v V_k^{\ast}\|_2 = \|\v V_k\|_2$,
and by Lemma \ref{lem:BiasBound1}, $\|\v V_k\|_2 \leq {\frac12} h[M\sqrt{hd}(2 +\sqrt{hM})\wedge4\sqrt{Md}]$. Therefore
\[
\|\v T_k\|_2 \leq \frac{{\frac12h[}\epsilon + M\sqrt{hd}(2 +\sqrt{hM})\wedge4\sqrt{Md}]}{1 + \frac12 hm\theta}.
\]
Since
\begin{align*}
1-\frac{1-\frac12 h(1-\theta)m}{1+\frac12 h\theta m} &= \frac{1 + \frac12 h\theta m - 1 + \frac12 h(1-\theta)m}{1+\frac12 h\theta m} = \frac{{\frac12 h m}}{1 + \frac12 h\theta m},\\
1-\frac{\frac12 h(1-\theta)M - 1}{\frac12 h\theta M + 1} &= \frac{1 + \frac12 h\theta M + 1 - \frac12 h (1-\theta)M}{\frac12 h \theta M + 1} = \frac{2 + \frac12 h (2\theta - 1)M}{\frac12 h \theta M + 1},
\end{align*}
applying the closed-form expression for the geometric series,
\[
\sum_{l=0}^{k-1} \rho^{k-l} \leq \frac{1}{1-\rho} = \max\left\lbrace \frac{1+\frac12 h \theta m}{{\frac12 h m}} , \frac{\frac12 h \theta M + 1}{2 + \frac12 h (2\theta - 1)M}\right\rbrace,
\]
and the result follows.
\end{proof}

\subsection{Proof of Theorem \ref{thm:GaussApprox} (Central Limit Theorem)}
\begin{proof}
The proof makes use of Laplace's method. From \cref{eq:overdamped_kernel} and the change of variables theorem, the Markov kernel $\tilde{p}_h(\v y \vert \v x)$ for the transition $\v X_k \mapsto \sqrt{h}(\v X_{k+1} - \v X_k^{\theta})$ is given by
\begin{multline}
\label{eq:KernPh1}
\tilde{p}_h(\v y \vert \v x) = (2\pi)^{-d/2} \det\left(\frac{1}{h}\m I
+ \frac{\theta}{2} \nabla^2 f(\v x^{\theta} + h^{-1/2} \v y)\right) \times
\\  \exp\left(-\frac{1}{2h}\left\lVert \v x^{\theta} - \v x + h^{-1/2} \v y
+ \frac{h\theta}{2}\nabla f(\v x^{\theta} + h^{-1/2} \v y) + \frac{h(1-\theta)}{2}
\nabla f(\v x)\right\rVert^2\right).
\end{multline}
Letting $q(\v y \vert \v x) = \phi(\v y; \v 0, \v\Sigma(\v x))$ where $\v\Sigma(\v x) = (4/\theta^2) \nabla^2 f(\v x^{\theta})^{-2}$, it suffices to show that $\tilde{p}_h(\v y \vert \v x) \to q(\v y \vert \v x)$ as $h \to \infty$,
pointwise in $\v y$.
Denoting $H_h(\v y) = h^{-1}\m I + \frac12 \theta\int_0^1 \nabla^2 f(\v x^{\theta} + t h^{-1/2} \v y) dt$, since $\theta \nabla f(\v x^{\theta}) = -(1-\theta) \nabla f(\v x)$, the exponent
of (\ref{eq:KernPh1}) becomes
\[
-\frac{1}{2h}\| \v x^{\theta} - \v x + \sqrt{h} H_h(\v y) \v y\|^2
= -\frac{\|\v x^{\theta} - \v x\|^2}{2 h} - \frac{(\v x^{\theta} - \v x)^{\top}
H_h(\v y) \v y}{\sqrt{h}} - \frac12 \v y^{\top} H_h(\v y)^2 \v y.
\]
Since $H_h(\v y)^2 \to \v\Sigma(\v x)^{-1}$ and the determinant term converges to $\det(\v\Sigma^{-1/2})$, the
result follows.
\end{proof}


\begin{thebibliography}{54}
\providecommand{\natexlab}[1]{#1}
\providecommand{\url}[1]{\texttt{#1}}
\expandafter\ifx\csname urlstyle\endcsname\relax
  \providecommand{\doi}[1]{doi: #1}\else
  \providecommand{\doi}{doi: \begingroup \urlstyle{rm}\Url}\fi

\bibitem[Anderson and Mattingly(2011)]{anderson2009}
David Anderson and Jonathan Mattingly.
\newblock A weak trapezoidal method for a class of stochastic differential
  equations.
\newblock \emph{Communications in mathematical sciences}, 9:\penalty0 301--318,
  03 2011.

\bibitem[Ascher(2008)]{ascher2008numerical}
Uri~M. Ascher.
\newblock \emph{{Numerical methods for evolutionary differential equations}}.
\newblock SIAM, 2008.

\bibitem[Ascher and Petzold(1998)]{ascher1998computer}
Uri~M. Ascher and Linda Petzold.
\newblock \emph{{Computer Methods for Ordinary Differential Equations and
  Differential-Algebraic Equations}}.
\newblock Other Titles in Applied Mathematics. Society for Industrial and
  Applied Mathematics (SIAM, 3600 Market Street, Floor 6, Philadelphia, PA
  19104), 1998.
\newblock ISBN 9781611971392.

\bibitem[Bendel and Mickey(1978)]{bendel1978population}
Robert~B. Bendel and M.~Ray Mickey.
\newblock {Population correlation matrices for sampling experiments}.
\newblock \emph{Communications in Statistics-Simulation and Computation},
  7\penalty0 (2):\penalty0 163--182, 1978.

\bibitem[Biscay et~al.(1996)Biscay, Jimenez, Riera, and
  Valdes]{biscay1996local}
R.~Biscay, J.~C. Jimenez, J.~J. Riera, and P.~A. Valdes.
\newblock Local linearization method for the numerical solution of stochastic
  differential equations.
\newblock \emph{Annals of the Institute of Statistical Mathematics},
  48\penalty0 (4):\penalty0 631--644, 1996.

\bibitem[Bishop and Tipping(2003)]{bishop2003bayesian}
Christopher~M. Bishop and Michael~E. Tipping.
\newblock Bayesian regression and classification.
\newblock \emph{Nato Science Series sub Series III Computer And Systems
  Sciences}, 190:\penalty0 267--288, 2003.

\bibitem[Boucheron et~al.(2013)Boucheron, Lugosi, and
  Massart]{boucheron2013concentration}
St{\'e}phane Boucheron, G{\'a}bor Lugosi, and Pascal Massart.
\newblock \emph{{Concentration inequalities: A nonasymptotic theory of
  independence}}.
\newblock Oxford university press, 2013.

\bibitem[Casella et~al.(2011)Casella, Roberts, and Stramer]{Casella2011}
Bruno Casella, Gareth Roberts, and Osnat Stramer.
\newblock Stability of partially implicit {L}angevin schemes and their {MCMC}
  variants.
\newblock \emph{Methodology and Computing in Applied Probability}, 13\penalty0
  (4):\penalty0 835--854, December 2011.

\bibitem[Chatfield et~al.(2010)Chatfield, Zidek, and
  Lindsey]{chatfield2010introduction}
Chris Chatfield, Jim Zidek, and Jim Lindsey.
\newblock \emph{An introduction to generalized linear models}.
\newblock Chapman and Hall/CRC, 2010.

\bibitem[Cheng and Bartlett(2018)]{cheng2018convergence}
Xiang Cheng and Peter Bartlett.
\newblock Convergence of {L}angevin {MCMC }in {KL}-divergence.
\newblock In \emph{Proceedings of Algorithmic Learning Theory}, volume~83,
  pages 186--211, 2018.

\bibitem[Cheng et~al.(2018{\natexlab{a}})Cheng, Chatterji, Abbasi-Yadkori,
  Bartlett, and Jordan]{cheng2018sharp}
Xiang Cheng, Niladri~S. Chatterji, Yasin Abbasi-Yadkori, Peter~L. Bartlett, and
  Michael~I. Jordan.
\newblock {Sharp Convergence Rates for Langevin Dynamics in the Nonconvex
  Setting}.
\newblock \emph{arXiv preprint arXiv:1805.01648}, 2018{\natexlab{a}}.

\bibitem[Cheng et~al.(2018{\natexlab{b}})Cheng, Chatterji, Bartlett, and
  Jordan]{cheng2017underdamped}
Xiang Cheng, Niladri~S. Chatterji, Peter~L. Bartlett, and Michael~I. Jordan.
\newblock Underdamped {L}angevin {MCMC}: A non-asymptotic analysis.
\newblock In S\'ebastien Bubeck, Vianney Perchet, and Philippe Rigollet,
  editors, \emph{Proceedings of the 31st Conference On Learning Theory},
  volume~75 of \emph{Proceedings of Machine Learning Research}, pages 300--323.
  PMLR, 06--09 Jul 2018{\natexlab{b}}.

\bibitem[Combettes and Pesquet(2011)]{combettes2011proximal}
Patrick~L. Combettes and Jean-Christophe Pesquet.
\newblock {Proximal splitting methods in signal processing}.
\newblock In \emph{Fixed-point algorithms for inverse problems in science and
  engineering}, pages 185--212. Springer, 2011.

\bibitem[Cottle and Thapa(2017)]{cottle2017linear}
Richard~W. Cottle and Mukund~N. Thapa.
\newblock \emph{Linear and Nonlinear Optimization}.
\newblock International Series in Operations Research \& Management Science.
  Springer New York, 2017.
\newblock ISBN 9781493970537.

\bibitem[Crane and Roosta(2019)]{crane2019dingo}
Rixon Crane and Fred Roosta.
\newblock {DINGO}: Distributed {N}ewton-type method for gradient-norm
  optimization.
\newblock In \emph{Advances in Neural Information Processing Systems}, pages
  9494--9504, 2019.

\bibitem[Dalalyan(2017{\natexlab{a}})]{dalalyan2017further}
Arnak~S. Dalalyan.
\newblock {Further and stronger analogy between sampling and optimization:
  Langevin Monte Carlo and gradient descent}.
\newblock \emph{arXiv preprint arXiv:1704.04752}, 2017{\natexlab{a}}.

\bibitem[Dalalyan(2017{\natexlab{b}})]{dalalyan2017theoretical}
Arnak~S. Dalalyan.
\newblock {Theoretical guarantees for approximate sampling from smooth and
  log-concave densities}.
\newblock \emph{Journal of the Royal Statistical Society: Series B (Statistical
  Methodology)}, 79\penalty0 (3):\penalty0 651--676, 2017{\natexlab{b}}.

\bibitem[Dalalyan and Karagulyan(2019)]{dalalyan2017user}
Arnak~S Dalalyan and Avetik Karagulyan.
\newblock User-friendly guarantees for the {L}angevin {M}onte {C}arlo with
  inaccurate gradient.
\newblock \emph{Stochastic Processes and their Applications}, 129\penalty0
  (12):\penalty0 5278--5311, 2019.

\bibitem[DasGupta(2011)]{dasgupta2011probability}
Anirban DasGupta.
\newblock \emph{{Probability for Statistics and Machine Learning: Fundamentals
  and Advanced Topics}}.
\newblock Springer Texts in Statistics. Springer New York, 2011.
\newblock ISBN 9781441996336.

\bibitem[Dua and Graff(2019)]{Dua:2019}
Dheeru Dua and Casey Graff.
\newblock {UCI} machine learning repository, 2019.
\newblock URL \url{http://archive.ics.uci.edu/ml}.

\bibitem[Durmus and Moulines(2017)]{durmus2017nonasymptotic}
Alain Durmus and Eric Moulines.
\newblock {Nonasymptotic convergence analysis for the unadjusted Langevin
  algorithm}.
\newblock \emph{The Annals of Applied Probability}, 27\penalty0 (3):\penalty0
  1551--1587, 2017.

\bibitem[Durmus and Moulines(2019)]{Durmus2016}
Alain Durmus and Eric Moulines.
\newblock High-dimensional {B}ayesian inference via the unadjusted {L}angevin
  algorithm.
\newblock \emph{Bernoulli}, 25\penalty0 (4A):\penalty0 2854--2882, 2019.

\bibitem[Girolami and Calderhead(2011)]{girolami2011riemann}
Mark Girolami and Ben Calderhead.
\newblock {Riemann manifold Langevin and Hamiltonian Monte Carlo methods}.
\newblock \emph{Journal of the Royal Statistical Society: Series B (Statistical
  Methodology)}, 73\penalty0 (2):\penalty0 123--214, 2011.

\bibitem[Golub and Van~Loan(2012)]{golub2012matrix}
Gene~H. Golub and Charles~F. Van~Loan.
\newblock \emph{{Matrix Computations}}, volume~3.
\newblock JHU Press, 4 edition, 2012.

\bibitem[Gretton et~al.(2012)Gretton, Borgwardt, Rasch, Sch{\"o}lkopf, and
  Smola]{gretton2012kernel}
Arthur Gretton, Karsten~M. Borgwardt, Malte~J. Rasch, Bernhard Sch{\"o}lkopf,
  and Alexander Smola.
\newblock A kernel two-sample test.
\newblock \emph{Journal of Machine Learning Research}, 13\penalty0
  (Mar):\penalty0 723--773, 2012.

\bibitem[Hansen(2003)]{hansen2003geometric}
Niels~Richard Hansen.
\newblock {Geometric ergodicity of discrete-time approximations to multivariate
  diffusions}.
\newblock \emph{Bernoulli}, 9\penalty0 (4):\penalty0 725--743, 2003.

\bibitem[Hastings(1970)]{Hastings1970}
Wilfred~K. Hastings.
\newblock {Monte Carlo} sampling methods using {Markov} chains and their
  applications.
\newblock \emph{Biometrika}, 57\penalty0 (1):\penalty0 97--109, 1970.
\newblock ISSN 00063444.

\bibitem[Ikeda and Watanabe(2014)]{ikeda2014stochastic}
Nobuyuki Ikeda and Shinzo Watanabe.
\newblock \emph{{Stochastic differential equations and diffusion processes}},
  volume~24.
\newblock Elsevier, 2014.

\bibitem[Kahaner et~al.(1989)Kahaner, Moler, Nash, and
  Forsythe]{kahaner1989numerical}
David Kahaner, Cleve~B. Moler, Stephen Nash, and George~E. Forsythe.
\newblock \emph{{Numerical Methods and Software}}.
\newblock Prentice-Hall series in computational mathematics. Prentice Hall,
  1989.

\bibitem[Kloeden and Platen(2013)]{kloeden2013numerical}
Peter~E. Kloeden and Eckhard Platen.
\newblock \emph{{Numerical solution of stochastic differential equations}},
  volume~23.
\newblock Springer Science \& Business Media, 2013.

\bibitem[Kolmogorov(1937)]{kolmogorov1937}
Andrei~N. Kolmogorov.
\newblock {Zur umkehrbarkeit der statistischen naturgesetze}.
\newblock \emph{Mathematische Annalen}, 113\penalty0 (1):\penalty0 766--772,
  1937.

\bibitem[Kopec(2014)]{kopec2014weak}
Marie Kopec.
\newblock {Weak backward error analysis for overdamped Langevin processes}.
\newblock \emph{IMA Journal of Numerical Analysis}, 35\penalty0 (2):\penalty0
  583--614, 2014.

\bibitem[Korattikara et~al.(2014)Korattikara, Chen, and
  Welling]{korattikara2014austerity}
Anoop Korattikara, Yutian Chen, and Max Welling.
\newblock Austerity in {MCMC} land: Cutting the {M}etropolis-{H}astings budget.
\newblock In \emph{International Conference on Machine Learning}, pages
  181--189, 2014.

\bibitem[Lambert(1991)]{lambert1991numerical}
John~Denholm Lambert.
\newblock \emph{{Numerical methods for ordinary differential systems: the
  initial value problem}}.
\newblock John Wiley \& Sons, Inc., 1991.

\bibitem[Maire et~al.(2018)Maire, Friel, and Alquier]{Maire2018}
Florian Maire, Nial Friel, and Pierre Alquier.
\newblock Informed sub-sampling {MCMC}: approximate {B}ayesian inference for
  large datasets.
\newblock \emph{Statistics and Computing}, pages 1--34, Jun 2018.

\bibitem[Mattingly et~al.(2002)Mattingly, Stuart, and
  Higham]{mattingly2002ergodicity}
Jonathan~C. Mattingly, Andrew~M. Stuart, and Desmond~J. Higham.
\newblock {Ergodicity for SDEs and approximations: locally Lipschitz vector
  fields and degenerate noise}.
\newblock \emph{Stochastic processes and their applications}, 101\penalty0
  (2):\penalty0 185--232, 2002.

\bibitem[McCullagh and Nelder(1989)]{mccullagh1989generalized}
Peter McCullagh and John~A. Nelder.
\newblock \emph{{Generalized linear models}}, volume~37.
\newblock CRC press, 1989.

\bibitem[Meyn and Tweedie(2012)]{meyn2012markov}
Sean~P. Meyn and Richard~L. Tweedie.
\newblock \emph{{Markov chains and stochastic stability}}.
\newblock Springer Science \& Business Media, 2012.

\bibitem[Muandet et~al.(2017)Muandet, Fukumizu, Sriperumbudur, and
  Sch{\"o}lkopf]{muandet2017kernel}
Krikamol Muandet, Kenji Fukumizu, Bharath Sriperumbudur, and Bernhard
  Sch{\"o}lkopf.
\newblock {Kernel mean embedding of distributions: A review and beyond}.
\newblock \emph{Foundations and Trends{\textregistered} in Machine Learning},
  10\penalty0 (1-2):\penalty0 1--141, 2017.

\bibitem[Nocedal and Wright(2006)]{nocedal2006numerical}
Jorge Nocedal and Stephen Wright.
\newblock \emph{{Numerical Optimization}}.
\newblock Springer Science \& Business Media, 2006.

\bibitem[Parikh and Boyd(2014)]{parikh2014proximal}
Neal Parikh and Stephen Boyd.
\newblock {Proximal Algorithms}.
\newblock \emph{Foundations and Trends{\textregistered} in Optimization},
  1\penalty0 (3):\penalty0 127--239, 2014.

\bibitem[Pereyra(2016)]{Pereyra2016}
Marcelo Pereyra.
\newblock Proximal {M}arkov {C}hain {M}onte {C}arlo algorithms.
\newblock \emph{Statistics and Computing}, 26\penalty0 (4):\penalty0 745--760,
  Jul 2016.
\newblock ISSN 1573-1375.

\bibitem[Robert and Casella(1999)]{robert1999monte}
Christian~P. Robert and George Casella.
\newblock \emph{{Monte Carlo Statistical Methods}}.
\newblock Springer texts in statistics. Springer, 1999.
\newblock ISBN 9780387987071.

\bibitem[Roberts and Rosenthal(2001)]{roberts2001optimal}
Gareth~O. Roberts and Jeffrey~S. Rosenthal.
\newblock {Optimal scaling for various Metropolis-Hastings algorithms}.
\newblock \emph{Statistical science}, 16\penalty0 (4):\penalty0 351--367, 2001.

\bibitem[Roberts and Tweedie(1996)]{roberts1996}
Gareth~O. Roberts and Richard~L. Tweedie.
\newblock Exponential convergence of {L}angevin distributions and their
  discrete approximations.
\newblock \emph{Bernoulli}, 2\penalty0 (4):\penalty0 341--363, 12 1996.

\bibitem[Rockafellar(1976)]{rockafellar1976monotone}
R~Tyrrell Rockafellar.
\newblock Monotone operators and the proximal point algorithm.
\newblock \emph{SIAM journal on control and optimization}, 14\penalty0
  (5):\penalty0 877--898, 1976.

\bibitem[Roosta et~al.(2018)Roosta, Liu, Xu, and Mahoney]{roosta2018newton}
Fred Roosta, Yang Liu, Peng Xu, and Michael~W. Mahoney.
\newblock {Newton-MR: Newton's Method Without Smoothness or Convexity}.
\newblock \emph{arXiv preprint arXiv:1810.00303}, 2018.

\bibitem[Roosta-Khorasani and Ascher(2015)]{roas1}
Farbod Roosta-Khorasani and Uri~M. Ascher.
\newblock {Improved bounds on sample size for implicit matrix trace
  estimators}.
\newblock \emph{Foundations of Computational Mathematics}, 15\penalty0
  (5):\penalty0 1187--1212, 2015.

\bibitem[Roosta-Khorasani and Mahoney(2018)]{roosta2018SSN}
Farbod Roosta-Khorasani and Michael~W. Mahoney.
\newblock {Sub-sampled Newton methods}.
\newblock \emph{Mathematical Programming}, Nov 2018.

\bibitem[Shao(2008)]{shao2008mathematical}
Jun Shao.
\newblock \emph{{Mathematical Statistics}}.
\newblock Springer Texts in Statistics. Springer New York, 2008.
\newblock ISBN 9780387217185.

\bibitem[S{\"u}li and Mayers(2003)]{suli2003introduction}
Endre S{\"u}li and David~F. Mayers.
\newblock \emph{{An introduction to numerical analysis}}.
\newblock Cambridge University Press, 2003.

\bibitem[Villani(2008)]{villani2008optimal}
C{\'e}dric Villani.
\newblock \emph{{Optimal Transport: Old and New}}, volume 338.
\newblock Springer Science \& Business Media, 2008.

\bibitem[Wibisono(2018)]{wibisono2018sampling}
Andre Wibisono.
\newblock Sampling as optimization in the space of measures: The {L}angevin
  dynamics as a composite optimization problem.
\newblock In S{\'{e}}bastien Bubeck, Vianney Perchet, and Philippe Rigollet,
  editors, \emph{Conference On Learning Theory, {COLT} 2018, Stockholm, Sweden,
  6-9 July 2018}, volume~75 of \emph{Proceedings of Machine Learning Research},
  pages 2093--3027. {PMLR}, 2018.

\bibitem[Xu et~al.(2017)Xu, Roosta, and Mahoney]{xuNonconvexTheoretical2017}
Peng Xu, Fred Roosta, and Michael~W Mahoney.
\newblock Newton-type methods for non-convex optimization under inexact
  {H}essian information.
\newblock \emph{Mathematical Programming}, pages 1--36, 2017.

\end{thebibliography}
\end{document}